\documentclass[journal]{IEEEtran}

 \usepackage[pdftex]{graphicx}

\usepackage{amsmath, amssymb, epsfig, graphicx, bm}
\usepackage{cases}
\usepackage{url}
\usepackage{float, subfigure}
\usepackage{color}
\usepackage{amsthm}
\usepackage{algorithmic,algorithm}
\newtheorem{theorem}{Theorem}

\newtheorem{lemma}{Lemma}
\newtheorem{remark}{Remark}

\newtheorem{assumption}{Assumption}
\newtheorem{assumption'}{Assumption'}

\allowdisplaybreaks[3]

\begin{document}

\title{Communication-Censored Distributed\\ Stochastic Gradient Descent}
%
%
%

\author{Weiyu~Li, Tianyi~Chen, Liping~Li, Zhaoxian~Wu, and Qing~Ling
\thanks{Weiyu Li is with School of Gifted Young, University of Science and Technology of China, Hefei, Anhui 230026, China.}
\thanks{Tianyi Chen is with Department of Electrical, Computer, and Systems Engineering, Rensselaer Polytechnic Institute, Troy, New York 12180, USA. }
\thanks{Liping Li is with Department of Automation, University of Science and Technology of China, Hefei, Anhui 230027, China.}
\thanks{Zhaoxian Wu and Qing Ling are with School of Data and Computer Science and Guangdong Province Key Laboratory of Computational Science, Sun Yat-Sen University, Guangzhou, Guangdong 510006, China. Qing Ling is supported in part by NSF China Grants 61573331 and 61973324, and Fundamental Research Funds for the Central Universities. Corresponding Email: lingqing556@mail.sysu.edu.cn.}}

%
%

%

\maketitle
\begin{abstract}
This paper develops a communication-efficient algorithm to solve the stochastic optimization problem defined over a distributed network, aiming at reducing the burdensome communication in applications such as distributed machine learning.
Different from the existing works based on quantization and sparsification, we introduce a communication-censoring technique to reduce the transmissions of variables, which leads to our communication-Censored distributed Stochastic Gradient Descent ({\bf CSGD}) algorithm.
Specifically, in CSGD, the latest mini-batch stochastic gradient at a worker will be transmitted to the server if and only if it is sufficiently informative. When the latest gradient is not available, the stale one will be reused at the server. To implement this communication-censoring strategy, the batch-size is increasing in order to alleviate the effect of stochastic gradient noise. Theoretically, CSGD enjoys the same order of convergence rate as that of SGD, but effectively reduces communication. Numerical experiments demonstrate the sizable communication saving of CSGD.
\end{abstract}

\IEEEpeerreviewmaketitle

\section{Introduction}
%
%
%
%
\IEEEPARstart{C}{onsidering} a distributed network with one server and $M$ workers,
we aim to design a communication-efficient algorithm to solve the following optimization problem
\begin{equation}
\min_{\mathbf{x}\in\mathbb{R}^d}~~\sum_{m=1}^M
\mathbb{E}_{\xi_m}\big[ f_m(\mathbf{x};\xi_m)\big],\label{eq:obj}
\end{equation}
where $\mathbf{x}$ is the optimization variable,
$\{f_m\}_{m=1}^M$ are smooth local objective functions with $f_m$ being kept at worker $m$,
and $\{\xi_m\}_{m=1}^M$ are independent random variables associated with distributions $\{\mathcal{D}_m\}_{m=1}^M$.

Problem \eqref{eq:obj} arises in a wide range of science and engineering fields, e.g.,
in distributed machine learning \cite{dean2012}. For distributed machine learning, there are two major drives for solving problems in the form of \eqref{eq:obj}: i) \emph{distributed computing resources} --- for massive and high-dimensional datasets, performing the training processes over multiple workers in parallel is more efficient than relying on a single worker; and, ii) \emph{user privacy concerns} --- with massive amount of sensors nowadays, distributively collected data may contain private information about end users, and thus keeping the computation at local workers is more privacy-preserving than uploading the data to central servers.
However, the communication between the server and the workers is one of the major bottlenecks of distributed machine learning. Indeed, reducing the communication cost is also a common consideration in popular machine learning frameworks, e.g., federated learning \cite{Konecny2016,Brendan2017,smith2017,Brendan2018}.

\subsection{Prior art}
Before discussing our algorithm,
we review several existing works for solving \eqref{eq:obj} in a distributed manner.

Finding the best communication-computation tradeoff has been a long-standing problem
in distributed consensus optimization \cite{Nedic2017,Berahas2018}, since it is critical to many important engineering problems in signal processing and wireless communications \cite{gg2016}.
For the emerging machine learning tasks,
the communication efficiency has been frequently discussed during the past decade \cite{zhang2013,Li2014,jordan2018},
and it attracts more attention when the notion of federated learning becomes popular
\cite{Konecny2016,Brendan2017,smith2017,Brendan2018}.
Many dual domain methods
have been demonstrated as efficient problem-solvers \cite{Tianbao2013,Jaggi2014},
which, nonetheless, require primal-dual loops and
own empirical communication-saving performance without theoretical guarantee.

In general, there are two different kinds of strategies to save communication cost.
On the one hand, due to the limited bandwidth in practice,
transmitting compressed information, which is called {\it quantization} \cite{Tang2018,Rao2019,Alistarh2017,bernstein2018icml} or {\it sparsification} \cite{stich2018nips,alistarh2018},
is an effective method to alleviate communication burden.
In particular, the quantized version of stochastic gradient descent (SGD) has been developed \cite{Alistarh2017,bernstein2018icml}.
On the other hand, instead of consistently broadcasting the latest information,
cutoff of some ``less informative'' messages is encouraged,
which results in the so-called {\it event-triggered control} \cite{Garcia2013,Aji2017} or {\it communication censoring} \cite{Li2019,Liu2019}.
Extending the original continuous-time setting \cite{Garcia2013}
to discrete-time \cite{Aji2017},
the work of \cite{lan2017} achieves a sublinear rate of convergence,
while the work of \cite{Liu2019} shows a linear rate
and its further extension \cite{Li2019} proves both rates of convergence.
However, the above three algorithms utilize the primal-dual loops, and only rigorously establish convergence without showing communication reduction.

To the best of our knowledge, lazily aggregated gradient (LAG) proposed in \cite{Tianyi2018}
is the most up-to-date method that provably converges and saves the communication.
However, the randomness in our problem \eqref{eq:obj}
deteriorates the deterministic LAG algorithm.
Simply speaking, an exponentially increasing batch-size and an exponentially decreasing censoring control-size are designed for our proposed stochastic algorithm to converge,
while LAG can directly obtain its deterministic gradient and does not introduce a controlling term in its threshold. Direct application of LAG to multi-agent reinforcement learning does involve stochasticity, but the convergence of the resultant LAPG approach is established in the weak sense, and communication reduction critically relies on the heterogeneity of distributed agents \cite{Tianyi2018-2}.

In our work, we generalize and strengthen the results in LAG \cite{Tianyi2018} and LAPG \cite{Tianyi2018-2} to a more challenging stochastic problem, with stronger convergence and communication reduction results.
Specifically, our convergence results hold in the \emph{almost sure} sense,
thanks to a novel design of time-varying batch-size in gradient sampling and control-size in censoring threshold. More importantly, our communication reduction is \emph{universal} in the sense that the heterogeneity characteristic needed to establish communication reduction in  \cite{Tianyi2018,Tianyi2018-2} is no longer a prerequisite in our work.

\subsection{Our contributions}
Though the celebrated SGD method \cite{bottou2010} can be applied to solving \eqref{eq:obj},
it requires iterative communication and is hence less advantageous in our setting.
Consider the SGD with dynamic batch-size \cite{Bottou2018} that starts from $\bar{\mathbf{x}}^{k-1}$ at iteration $k$.
After receiving the variable $\bar{\mathbf{x}}^{k-1}$ from the server, every worker $m$ samples a batch of independent and identically distributed (i.i.d.) stochastic gradients
$\{\nabla f_m(\bar{\mathbf{x}}^{k-1};\xi_m^{k,b})\}_{b=1}^{B^k}$ with a batch-size $B^k$,
and then sends the sample mean $\bar{\nabla}_m^k:=\frac1{B^k}\sum_{b=1}^{B^k}\nabla f_m(\bar{\mathbf{x}}^{k-1};\xi_m^{k,b})$ back to the server, which aggregates all the means and performs the \textbf{SGD update} with step-size $\alpha$ as
\begin{equation}
\bar{\mathbf{x}}^{k}=\bar{\mathbf{x}}^{k-1}-\alpha \bar\nabla^k:=\bar{\mathbf{x}}^{k-1}-\alpha \sum_{m=1}^M\bar{\nabla}_m^k.\label{SGD}
\end{equation}
Therein, every worker is required to upload the latest locally averaged gradient at every iteration, which is rather expensive in communication.

To maintain the desired properties of SGD and overcome its limitations,
we design our communication-Censored distributed SGD (CSGD) method,
which leverages the {\em communication-censoring} strategy. Consider a starting point $\mathbf{x}^{k-1}$ at iteration $k$. As in SGD, every worker $m$ samples a batch of independent and identically distributed (i.i.d.) stochastic gradients
$\{\nabla f_m({\mathbf{x}}^{k-1};\xi_m^{k,b})\}_{b=1}^{B^k}$ with a batch-size $B^k$,
and then calculates the sample mean ${\nabla}_m^k:=\frac1{B^k}\sum_{b=1}^{B^k}\nabla f_m({\mathbf{x}}^{k-1};\xi_m^{k,b})$.
Seeking a desired communication-censoring strategy,
we are interested in the distance between the calculated gradient $\nabla_m^k$ at worker $m$ and
the latest uploaded one before iteration $k$ starts, denoted by $\hat\nabla_m^{k-1}$.
While other distance metrics are also available, we consider the distance in terms of $\|\nabla_m^k-\hat\nabla_m^{k-1}\|^2$,
where $\|\cdot\|$ denotes the $2$-norm of a vector.
If $\|\nabla_m^k-\hat\nabla_m^{k-1}\|^2$ is below a censoring threshold $\tau^k$, $\nabla_m^k$ is regarded as less informative and will not be transmitted.
Therefore, the latest uploaded gradient for worker $m$ up to iteration $k$, denoted as $\hat\nabla_m^k$, follows the updating rule

\begin{align} \label{equation:phi}
 \hat\nabla_m^k  =
  \begin{cases}
    \nabla_m^k, \quad\quad &\|\nabla_m^k-\hat\nabla_m^{k-1}\|^2>\tau^k, \\
    \hat\nabla_m^{k-1}, \quad\quad &{\rm otherwise}.
  \end{cases}
\end{align}
Then the server aggregates the latest received gradients in $\hat{\nabla}^k:= \sum_{m=1}^M\hat\nabla_m^k$ and performs the \textbf{CSGD update} with the step-size $\alpha$, that is,
\begin{equation}
\mathbf{x}^{k}=\mathbf{x}^{k-1}-\alpha \hat{\nabla}^k
=\mathbf{x}^{k-1}-\alpha \sum_{m=1}^M\hat\nabla_m^k,\label{CSGD}
\end{equation}
Specifically, in \eqref{equation:phi} we use the following censoring threshold
\begin{equation}
\tau^k:=\frac{1}{M^2}\bigg(\frac{w}D\sum_{d=1}^D\|\hat\nabla^{k-d}\|^2
+ \sigma^k \bigg),\label{threshold}
\end{equation}
where $\{\|\hat\nabla^{k-d}\|\}_{d=1}^D$ are $2$-norms of recent $D$ aggregated gradients with $\|\hat\nabla^{-d}\|=0$ for $d>0$,
$w$ is a weight representing the confidence of the censoring threshold, and $ \sigma^k $ controls the randomness of the stochastic part that we call {\it control-size}.
The adaptive threshold consists of a scaling factor $\frac1{M^2}$ and the sum of two parts.
The first part learns information from the previous $D$ updates,
while the second part helps alleviate the stochastic gradient noise.

Building upon this innovative censoring condition, our main contributions can be summarized as follows.
\begin{enumerate}
\item[\bf c1)] We propose communication-censored distributed SGD (CSGD) with dynamic batch-size that achieves the same order of convergence rate as the original SGD.
\item[\bf c2)] CSGD provably saves the total number of uploads to reach the targeted accuracy relative to SGD.
\item[\bf c3)] We conduct extensive experiments to show the superior performance of the proposed CSGD algorithm.
\end{enumerate}

\section{CSGD development}\label{sec:alg}
In this section, we introduce CSGD and provide some insights behind its threshold design in \eqref{threshold}.
In CSGD, at the beginning of iteration $k$, the server broadcasts its latest variable and threshold to all workers. With the considerations of data privacy and uploading burden, each worker locally computes an estimate of its gradient with batch-size $B^k$ and then decides whether to upload the fresh gradient or not. Specifically, the worker's upload will be skipped, if and only if $\|\nabla_m^k-\hat\nabla_m^{k-1}\|^2\le\tau^k$. When such a communication skipping happens, we say that the worker is {\it censored}. At the end of iteration $k$, the server only receives the latest uploaded gradients, and updates its variable via \eqref{CSGD} and the censoring threshold $\tau^{k+1}$ via \eqref{threshold}, using the magnitudes of $D$ recent updates given by $\{\|\hat\nabla^{k-d}\|\}_{d=0}^{D-1}$.
We illustrate CSGD in Figure \ref{fig:frame} and summarize it in Algorithm \ref{alg:censor}.
\begin{figure}[ht!]
\vspace{-.8cm}
\centering
\includegraphics[height=5.5cm]{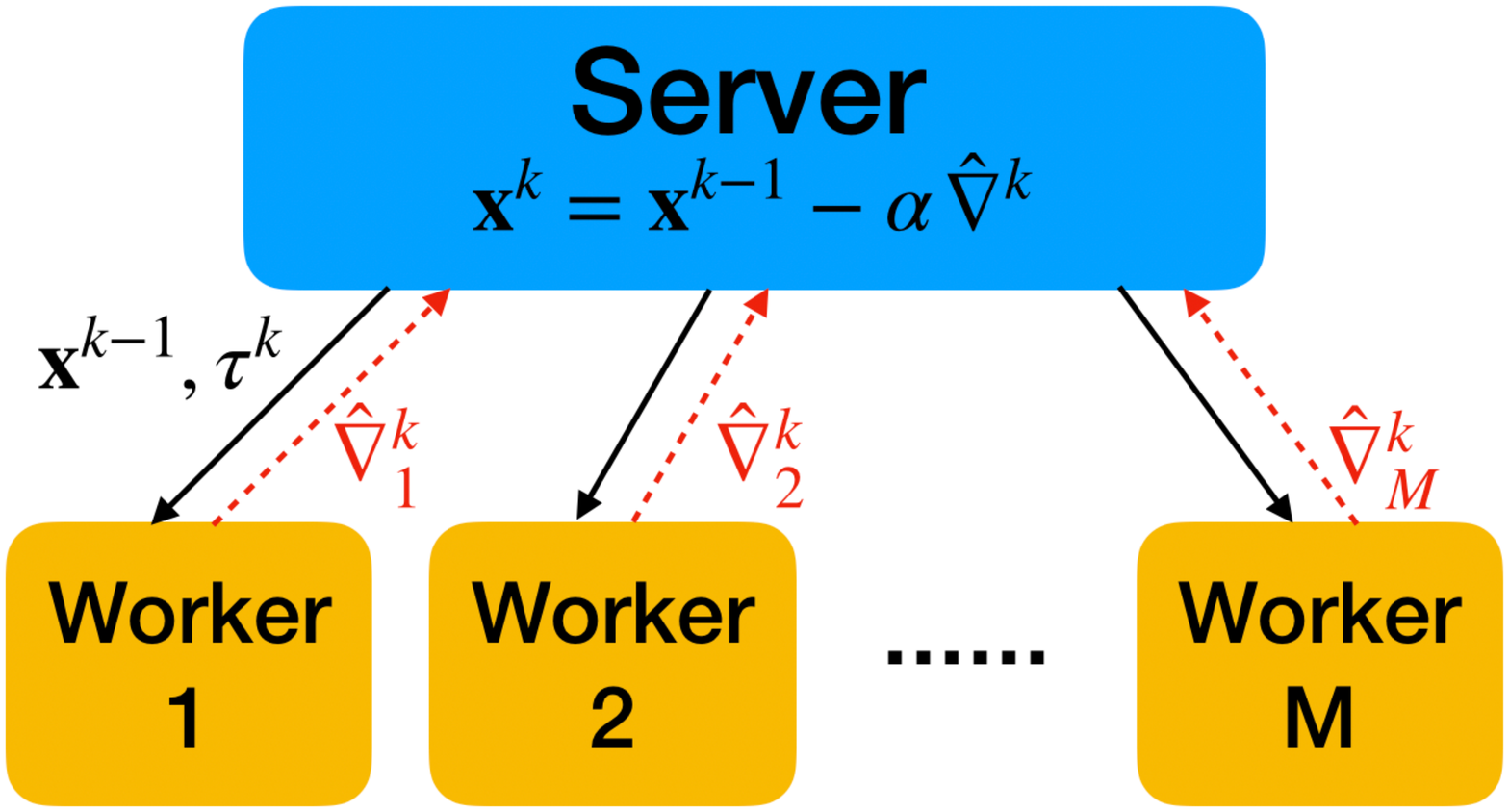}
\vspace{-1cm}
  \caption{Illustration of CSGD.}
\label{fig:frame}
\vspace{-0.5cm}
\end{figure}

\subsection{CSGD parameters}
If we choose the parameters properly, our proposed framework is general in the sense that it also recovers several existing algorithms. For the deterministic optimization problem, LAG \cite{Tianyi2018} which sets $w<1$ and $\sigma^k=0$ in \eqref{threshold} guarantees communication saving compared with the original gradient descent (i.e. $w=0,\sigma^k=0$). Note that in the deterministic case, all data are used at every iteration such that the control-size $\sigma^k$ designed to handle the randomness is not necessary and the batch-size $B^k$ is no longer a hyper-parameter. For the stochastic optimization task, setting $w=0,\sigma^k=0$ in \eqref{threshold} recovers the SGD with dynamic batch-size \cite{Bottou2018,Hao2019}.
We give some interpretations of the parameters in CSGD as follows; see also Table \ref{table1}.

\vspace{0.1cm}
\noindent\textbf{The step-size $\alpha$ and the batch-size $B^k$.}
In recent works \cite{Bottou2018,Hao2019}, SGD with constant step-size and exponentially-increasing batch-size has been studied. It achieves the $O(1/k)$ accuracy with $O(\log k)$ iterations and $O(k)$ samples of stochastic gradients. Intuitively speaking, larger step-size $\alpha$ leads to faster convergence, but requires a faster increasing rate of batch-size (which depends on $\alpha$) to control the bias from the stochastic gradient sampling. Then in total, the sampling time is in the same order regardless of the magnitude of $\alpha$.
Nonetheless, the choice of $\alpha$ cannot be arbitrary; extremely large step-size learns from the current stochastic gradient too much, thus deteriorates the convergence.

In our analysis, choosing the increasing rate of $B^k$ larger than a lower bound depending on $\alpha$ will result in a convergence rate depending only on $\alpha$, which is consistent with previous SGD works.

\vspace{0.1cm}
\noindent\textbf{The control-size $\sigma^k$.}
The term $\sigma^k$ has two implications.
\begin{enumerate}
\item It excludes some noisy uploads. When the worker takes the mean of $B^k$ stochastic gradients, the variance of the mean shrinks to $\frac1{B^k}$ of the variance of one stochastic gradient, if the variance exists. Thus, $\sigma^k$ decreasing no faster than $\frac1{B^k}$ helps the threshold to make effect in the long term.
\item As a tradeoff, the control-size may slow down the convergence. If it decreases in an extremely slow rate, the censoring threshold will be hard to reach, and the server will use the inaccurate stale gradient for a long time before receiving a fresh gradient, which affects the rate of convergence.
\end{enumerate}

In the theoretical analysis, we will theoretically show that if $\sigma^k$ decreases properly at a rate similar to those of $\frac1{B^k}$ and the objective, then CSGD converges at a comparable rate to that of SGD, but with improved communication-efficiency.

\begin{algorithm}[t]
	{
		\caption{Communication-Censored Distributed SGD}
		\begin{algorithmic}[1]\label{alg:censor}
			\REQUIRE $\alpha$, $\{B^k\}_{k=1}^K$, $\{\sigma^k\}_{k=1}^K$.\\
			\hspace{-0.5cm}\textbf{Initialize:} $\mathbf{x}^0$, $\{\hat{\nabla}_m^0\}_{m=1}^M$, $\tau^1$.
			\FOR {iterations $k = 1,2, \ldots$}
			\STATE Server broadcasts $\mathbf{x}^{k-1},\tau^k$.
			\FOR {Worker $m=1,\cdots,M$}
			\STATE Sample stochastic gradients $\{\nabla f_m(\mathbf{x}^{k-1};\xi_m^{k,b})\}_{b=1}^{B^k}$.
			\STATE Compute $\nabla_m^k=\frac1{B^k}\sum_{b=1}^{B^k}\nabla f_m(\mathbf{x}^{k-1};\xi_m^{k,b}).$
			\IF {$\|\nabla_m^k-\hat\nabla_m^{k-1}\|^2>\tau^k$}
			\STATE Worker $m$ uploads $\nabla_m^k$ to the server.
			\STATE Set $\hat\nabla_m^k = \nabla_m^k$ on the server and worker $m$.
			\ELSE \STATE {Worker $m$ does not upload.}
			\STATE Set $\hat\nabla_m^k = \hat\nabla_m^{k-1}$ on the server and worker $m$.
			\ENDIF
			\ENDFOR
			\STATE Server updates the model $\mathbf{x}^k$ via \eqref{CSGD} and $\tau^{k+1}$ via \eqref{threshold}.
			\ENDFOR
	\end{algorithmic}}
\end{algorithm}

\vspace{0.1cm}
\noindent\textbf{The confidence time $D$ and confidence weight $w$.}
Those two parameters bound how much historic information we leverage in CSGD. First, $D$ is regarded as a {\it confidence time}. Once a newly calculated local gradient is uploaded, we are confident that it will be a good approximation of the gradients in the consecutive $D$ iterations from now on. Therefore, we prefer using it to update variables for no less than $D$ times, instead of uploading a fresher gradient. In fact, the communication-saving property proved in the next section is motivated by the intuition that the upload is as sparse as no more than once in $D$ consecutive iterations.
Meanwhile, we multiply a weight $w<1$ to the historic gradients, with the consideration of lessening the impact of historic errors.

Theoretically, we specify $w=\frac1{60}$ to simplify the threshold, and constrain on step-size and batch-size such that any large $D$ is able to work as a confidence time.

\subsection{Motivation of the censoring threshold $\tau^k$}

For brevity, we stack the random variables into $\xi=[\xi_1;\ldots;\xi_M]$,
and define $f(\mathbf{x};\xi)= \sum_{m=1}^M f_m(\mathbf{x};\xi_m)$,
$F_m(\mathbf{x})=\mathbb{E}_{\xi_m}\big[ f_m(\mathbf{x};\xi_m)\big]$,
and $F(\mathbf{x})=\sum_{m=1}^M F_m(\mathbf{x})$.

The following lemma bounds how much the objectives in CSGD and SGD descend after one update.
\begin{lemma}[Objective descent] \label{lem:descent}
Suppose the gradient of the objective function $F(\mathbf{x})$ is $L$-Lipschitz continuous, then for SGD iteration \eqref{SGD}, we have for any $\bar\epsilon>0$ that
\begin{align}
&F(\bar{\mathbf{x}}^{k})-F(\bar{\mathbf{x}}^{k-1})\notag\\
\le&-\!\alpha\bigg(1 -\frac{\bar\epsilon}2 -(1+\bar\epsilon)\frac{L}2\alpha\bigg)\| \nabla F(\bar{\mathbf{x}}^{k-1})\|^2\notag\\ &+\frac{\alpha }{2\bar\epsilon}\|\nabla F(\bar{\mathbf{x}}^{k-1})- \bar{\nabla}^k\|^2\notag\\
:= &\bar{\Delta}^k(\bar{\mathbf{x}}^{k-1},\bar\epsilon).\label{lem1}
\end{align}
Likewise for CSGD iteration \eqref{CSGD}, we have for any $\epsilon>0$ that
\begin{align}
&F({\mathbf{x}}^{k})-F(\mathbf{x}^{k-1})\notag\\
\le&-\alpha\bigg(1 -\frac\epsilon2 -(1+\epsilon)\frac{L}2\alpha\bigg)\| \nabla F(\mathbf{x}^{k-1})\|^2\notag\\ &+\frac{\alpha }{\epsilon}\|\nabla F(\mathbf{x}^{k-1})- \nabla^k\|^2+\alpha M^2\bigg(\frac{1}{\epsilon}+(1+\frac1{\epsilon})\frac{L}2\alpha \bigg)\tau^k
\notag\\
:=&\Delta^k({\mathbf{x}}^{k-1},\epsilon).\label{lem2}
\end{align}
\end{lemma}

\begin{table*}
\centering
\begin{tabular}{|c|c|c|}
\hline
{\bf Notation} & {\bf Description} & {\bf Theoretically suggested value} \\ \hline
$\alpha$        &    step-size         &   $\alpha \le\min\{\frac{3}{2D\mu},\frac1{3L},\frac1{6\sqrt{5 \max_mL_m}MD}\}$ \\ \hline
$B^k$            &    batch-size       &    $B^k\ge B^0(1-\eta_1)^{-k}$   \\ \hline
$\sigma^k$    &    control-size     &    $\sigma^k=\sigma^0(1-\eta_2)^k$                 \\ \hline
$D$               &  confidence time &   $D\ge2$ (Polyak-\L ojasiewicz) or $D\ge10$ (nonconvex) \\ \hline
$w$         	     &   confidence weight &   $w=\frac{1}{60}$                 \\ \hline
\end{tabular}
\vspace{0.3cm}
\caption{Important parameters and their theoretical settings that provably save communication.}
\label{table1}
\vspace{-0.3cm}
\end{table*}

Recall the confidence time interpretation of the constant $D$ in \eqref{threshold}. Ideally, in CSGD, an uploaded gradient will be used for at least $D$ iterations, and thus the number of communication reduces to at most $\frac1D$ of the uncensored SGD. At the same time, the objective may descend less in CSGD relative to SGD. Nonetheless, conditioned on the same starting point $\bar{\mathbf{x}}^{k-1}={\mathbf{x}}^{k-1}$ at iteration $k$, if the objective descents of CSGD and SGD satisfy
\begin{equation}
\frac{-\Delta^k({\mathbf{x}}^{k-1},\epsilon)}{-\bar{\Delta}^k({\mathbf{x}}^{k-1},\bar\epsilon)}\ge \frac1D\label{com-aim}
\end{equation}
then CSGD still outperforms SGD in terms of communication efficiency.
Equivalently, we write \eqref{com-aim} as
\begin{align}
\!\!\!\tau^k\!\le\! \frac1{M^2}\left(w\| \nabla F(\mathbf{x}^{k-1})\|^2 + c\|\nabla F(\mathbf{x}^{k-1})- \nabla^k\|^2 \right),\label{eq:tauk-bound}
\end{align}
where $$w=\frac{\left(1 -\frac\epsilon2 -(1+\epsilon)\frac{L}2\alpha\right)-\frac1D\left(1 -\frac{\bar\epsilon}2 -(1+\bar\epsilon)\frac{L}2\alpha\right)}{\frac{1}{\epsilon}+(1+\frac1{\epsilon})\frac{L}2\alpha}$$ and $$c=\frac{\frac1{2D\bar\epsilon}-\frac1\epsilon}{\frac{1}{\epsilon}+(1+\frac1{\epsilon})\frac{L}2\alpha}$$ are two constants.

Intuitively, larger $\tau^k$ increases the possibility of censoring communication.
However, the right-hand side of \eqref{eq:tauk-bound} is not available at the beginning of iteration $k$, since we know neither $\nabla F(\mathbf{x}^{k-1})$ nor $\nabla^k$. Instead, we will approximate $\|\nabla F(\mathbf{x}^{k-1})\|^2$ using the aggregated gradients in the recent $D$ iterations, that is,
\begin{align*}
\|\nabla F(\mathbf{x}^{k-1})\|^2\approx\frac1D\sum_{d=1}^D\|\hat\nabla^{k-d}\|^2.
\end{align*}
Further controlling
$c\|\nabla F(\mathbf{x}^{k-1})-\nabla^k\|^2$ by $\sigma^k$, \eqref{eq:tauk-bound} becomes
\begin{equation*}
\tau^k=\underbrace{\frac{1}{M^2}}_\text{scaling}\bigg(w\cdot\underbrace{\frac{1}D\sum_{d=1}^D\|\hat\nabla^{k-d}\|^2}_\text{gradient approximation}
+\underbrace{ \sigma^k }_\text{error control}\bigg),
\end{equation*}
which leads to the CSGD threshold \eqref{threshold}.


\section{Theoretical results}\label{sec:thm}

In this section, we study how the introduction of censoring in CSGD affects the convergence as well as the communication, compared to the uncensored SGD. The proofs are given in the appendix.
Before presenting our theoretical results, we first provide the following sufficient conditions.
\begin{assumption}[Aggregate function]
The aggregate function $f(\mathbf{x};\xi)$ and its expectation $F(\mathbf{x})$ satisfy:
\begin{enumerate}
    \item Smoothness: $\nabla F(\mathbf{x})$ is $L$-Lipschitz continuous..
    \item Bounded variance: for any $\mathbf{x}\in\mathbb{R}^d$, there exists $G<\infty$ such that
    \begin{equation}
        \mathbb{E}\big\|\nabla f(\mathbf{x};\xi)-\nabla F(\mathbf{x})\big\|^2\le G^2.\label{eq:var-opt}
    \end{equation}
\end{enumerate}
\label{ass:global}
\end{assumption}

\begin{assumption}[Local functions]
Two conditions on the function per worker are given as follows.
\begin{enumerate}
    \item Smoothness: for each $m$, $\nabla F_m(\mathbf{x})$ is $L_m$-Lipschitz continuous.
    \item Bounded variance: for any $\mathbf{x}\in\mathbb{R}^d$ and $m$, there exists $G_m<\infty$ such that
\begin{equation}
    \mathbb{E}\big\|\nabla f_m(\mathbf{x};\xi_m)-\nabla F_m(\mathbf{x})\big\|^2\le G_m^2.\label{eq:sub-gau}
\end{equation}
\end{enumerate}
\label{ass:local}
\end{assumption}

Notice that Assumption \ref{ass:local} is sufficient for Assumption \ref{ass:global} to hold with $L=\sum_{m=1}^M L_m$, and the independence of $\{\xi_m\}_{m=1}^M$ leads to $G^2=\sum_{m=1}^M G_m^2$.

\subsection{Polyak-\L ojasiewicz case}

In the first part, we will assume the Polyak-\L ojasiewicz condition \cite{Karimi2016},
which is generally weaker than strong convexity, or even convexity.
\begin{assumption}[Polyak-\L ojasiewicz condition]
There exists a constant $\mu>0$ such that for any $\mathbf{x}$, we have
\begin{equation}
    2\mu \left(F(\mathbf{x})-F^*\right)\le\|\nabla F(\mathbf{x})\|^2,\label{PL}
\end{equation}
where $F^*$ is the minimum of \eqref{eq:obj}.
\label{ass:convex}
\end{assumption}

Define the {\bf Lyapunov function} for CSGD as
\begin{equation}\label{Lyapunov}
V^k:=F(\mathbf{x}^k)-F^*
+\sum_{d=1}^D\beta_d\|\hat\nabla^{k-d+1}\|^2,
\end{equation}
where $\{\beta_d=\frac{D+1-d}{9D}\alpha\}_{d=1}^D$ is a set of constant weights.
Analogously, the \textbf{Lyapunov function} for uncensored SGD is defined as
\begin{equation}\label{Lyapunov-sgd}
\bar{V}^k:=F(\bar{\mathbf{x}}^k)-F^*
+\sum_{d=1}^D\beta_d\|\bar\nabla^{k-d+1}\|^2.
\end{equation}

The following theorem guarantees the almost sure (a.s.) convergence of CSGD.
\begin{theorem}[Almost sure convergence] \label{thm:as}
Under Assumptions \ref{ass:global} and \ref{ass:convex}, if we choose $w=\frac1{60}$, $\alpha \le \min\{\frac{3}{2D\mu}, \frac1{3L}\}$, and $\sum_{k=1}^\infty\sigma^k$, $\sum_{k=1}^\infty\frac1{B^k}<\infty$, then it follows that
\begin{align}\label{eq:thm1}
\lim_{k\rightarrow \infty} V^k=0~~~{\rm and}~~~\lim_{k\rightarrow \infty} F(\mathbf{x}^k)=F^*~~~{\rm a.s}.
\end{align}
\end{theorem}

In addition to the asymptotic convergence, we establish the linear convergence rate for CSGD.
\begin{theorem}[Convergence rate] \label{thm:rate}
Under the same assumptions and parameter settings as those in Theorem \ref{thm:as},
further denote $\rho=\frac13\mu\alpha$ and assume
\begin{equation}\label{stepsize}
B^k\ge B^0(1-\eta_1)^{-k},~~~\sigma^k\le\sigma^0(1-\eta_2)^k
\end{equation}
for some $\eta_1,\eta_2>\rho$. Then conditioned on the same initial point $\mathbf{x}^0$, we have
\begin{equation}
\hspace{-.4em}\mathbb{E}[V^k|\mathbf{x}^0]\le C_{\rm CSGD}(1-\rho)^k,\  \mathbb{E}[\bar{V}^k|\mathbf{x}^0]\le C_{\rm SGD}(1-\rho)^k,\label{Vbound}
\end{equation}
where $C_{\rm CSGD}=V^0+\frac{\alpha(1-\rho)}3(\frac{10\sigma^0}{\eta_2-\rho}+\frac{7G^2}{B^0(\eta_1-\rho)}),$ $C_{\rm SGD}= V^0+\frac{7\alpha(1-\rho)G^2}{3B^0(\eta_1-\rho)}$ are two constants.
\end{theorem}

Theorem \ref{thm:rate} implies that even if CSGD skips some communications, its convergence rate is still in the same order as that of the original SGD.
We define the \textbf{$\nu$-iteration complexity} of CSGD as $K=\min\{k:C_{\rm CSGD}(1-\rho)^k\le\nu\}$ and \textbf{$\nu$-communication complexity} as the total number of uploads up to iteration $K$. Analogously defining the complexities for SGD,
we will compare the communication complexities of these two approaches in the following theorem.
\begin{theorem}[Communication saving] \label{thm:sav}
Under Assumptions \ref{ass:local} and \ref{ass:convex},
set $w=\frac1{60}$,
$\alpha \le\min\{\frac{3}{2D\mu},\frac1{3L},\frac1{6\sqrt{5 \max_mL_m}MD}\}$.
Further, denote $\rho=\frac13\mu\alpha$ and assume
\begin{equation}\label{batchsize}
B^k\ge B^0(1-\eta_1)^{-k},\text{ and }\sigma^k=\sigma^0(1-\eta_2)^k
\end{equation}
for some $\eta_1>\eta_2>\rho$ and $B^0\ge\frac{6M^2(1-\eta_1)\sum_{m=1}^MG_m^2}{\sigma^0(\eta_1-\eta_2)(1-\eta_2)^D\delta}$, where $\delta>0$ is a given probability.
Then with probability at least $1-\delta$,
each worker updates at most once in every $D$ consecutive iterations.
In addition, CSGD will save communication in the sense of having less communication complexity, when
\begin{equation}\label{eq:save-cond}
D\ge2,~~~ K\ge6.
\end{equation}
\end{theorem}
\begin{remark}
The batch-size rule in \eqref{batchsize} is commonly used in SGD algorithms with dynamic batch-size to guarantee convergence \cite{Bottou2018,Hao2019}. On the other hand, censoring introduces a geometrically convergent control-size $\sigma^k$, which leads to the same convergence rate as SGD, but provably improves communication-efficiency.
\end{remark}

In short, Theorem \ref{thm:sav} implies that with high probability, if we properly choose the parameters and run CSGD more than a given number of iterations, then the censoring strategy helps CSGD save communication. Intuitively, a larger $D$ cuts off more communications, while it slows down the linear rate of convergence since $\rho=\frac13\mu\alpha\le\frac1{2D}$.

Compared to LAG \cite{Tianyi2018} and LAPG \cite{Tianyi2018-2}, whose objective functions are not stochastic, our convergence results in Theorem \ref{thm:as} hold in the almost sure sense, and our communication reduction in Theorem \ref{thm:sav} is universal. That is to say, with a smaller step-size, the heterogeneity characteristic needed to establish communication reduction in LAG and LAPG is no longer a prerequisite in our work. Note that for both CSGD and SGD with dynamic batch-size \cite[Theorem 5.3]{bottou2016} \cite{Hao2019}, the magnitude of step-size does not affect the order of the overall number of stochastic gradient calculations to achieve the targeted accuracy. Specifically, since $\eta_1$ approaches $\rho$ as shown in Theorem \ref{thm:rate}, up to iteration $K$ (the $\nu$-iteration complexity of CSGD), the number of needed samples is
\begin{align}
\sum_{k=1}^K B^k&=\sum_{k=1}^K B^0(1-\eta_1)^{-k}\notag\\
&\approx\sum_{k=1}^K B^0(1-\rho)^{-k}\notag\\&=O((1-\rho)^{-K})=O(\nu^{-1}),
\end{align}
where the order is regardless of the magnitude of $\alpha$. 
Therefore, different from using an optimal (possibly large) step-size in existing algorithms like LAG and LAPG, it is reasonable to set the step-sizes in SGD and CSGD the same small value, which leads to our universal communication reduction result.

\begin{remark}
For simplicity, in the proof, we set the constants in Lemma \ref{lem:descent} as $\bar\epsilon=\epsilon=\frac12$. Keeping $\bar\epsilon=\epsilon$ gives the same linear rate of convergence. Yet, with different values, we may achieve better results but it is not the main focus here.
\end{remark}

\subsection{Nonconvex case}
While CSGD can achieve linear convergence rate under the Polyak-\L ojasiewicz condition, many important learning problems such as deep neural network training do not satisfy such a condition.
Without Assumption \ref{ass:convex}, we establish more general results that also work for a large family of nonconvex functions.
\begin{theorem}[Nonconvex case]\label{thm:nonconvex}
Under Assumption \ref{ass:local} and the same parameter settings in Theorem \ref{thm:sav},
then with probability at least $1-\delta$, we have
\begin{equation}\label{eq:min-nonconvex}
\min_{1\le k\le K}\|\nabla F(\mathbf{x}^k)\|^2=o\big(\frac1K\big),\ \min_{1\le k\le K}\|\hat\nabla^k\|^2=o\big(\frac1K\big),
\end{equation}
and each worker uploads at most once in $D$ consecutive iterations.
As a consequence, if we evaluate $\nu$-iteration complexity by $\min_{1\le k\le K}\|\nabla F(\mathbf{x}^k)\|^2$ and correspondingly define $\nu$-communication complexity as the total number of communications up to its iteration complexity time, then CSGD will save communication in the sense of having less communication complexity as long as
\begin{equation}\label{eq:save-cond2}
D\ge 10,\quad K\ge448.
\end{equation}
\end{theorem}

\section{Numerical experiments}
To demonstrate the merits of our proposed CSGD, especially the two-part design of the censoring threshold,
we conduct experiments on four different problems, least squares on a synthetic dataset,
softmax regression on the MNIST dataset \cite{lecun1998mnist} and
logistic regression on the Covertype dataset \cite{covtype}, and deep neural network training on the CIFAR-10 dataset \cite{cifar10}. The source codes are available at \url{https://github.com/Weiyu-USTC/CSGD}. All experiments are conducted using Python 3.7.4  on a computer with Intel i5 CPU @ 2.3GHz. We simulate one server and ten workers. To benchmark CSGD, we consider the following approaches.

\noindent
$\triangleright$ \textbf{CSGD}: our proposed method with update \eqref{CSGD}.

\noindent
$\triangleright$ \textbf{LAG-S}: directly applying the LAG \cite{Tianyi2018} censoring condition to the stochastic problem, which can be viewed as CSGD with zero control-size.

\noindent
$\triangleright$ \textbf{SGD}: update \eqref{SGD}, which can be viewed as CSGD with censoring threshold $0$.
\vspace{0.05cm}

In practice, when the batch-size is larger than the number of samples (denoted as $\bar{B}$), a worker can get all the samples, thus there is no more need of stochastic sampling and averaging.
Therefore in the experiments, unless otherwise specified, the batch-size and censoring threshold are calculated via
\begin{equation}
\hspace{0em}\begin{cases}
    B^k=\min\left\{ \lceil B^0(1-\eta_1)^{-k}\rceil, \bar{B}\right\}, \\
    \tau^k=\frac{w}{M^2D\alpha^2}\sum_{d=1}^D\|\mathbf{x}^{k-d}-\mathbf{x}^{k-d-1}\|^2+\frac{\sigma^0}{M^2}(1-\eta_2)^k,
\end{cases} \label{eq:tauk-exp}
\end{equation}
where $B^0,\eta_1,\eta_2,\sigma^0,D$ are set the same for three algorithms, while $w,\sigma^0$ are parameters depending on which method we use. Specifically,
$w$ in CSGD and LAG-S is set as $\frac1{60}$ according to our theoretical analysis, while $w=0$ in SGD. The initial control-size $\sigma^0$ is manually tuned to give proper performance in the first few iterations of CSGD, and is $0$ in LAG-S and SGD.
In all the experiments, we choose $D=10$, since it works for both Polyak-\L ojasiewicz and nonconvex cases in the theorems.

We tune the parameters by the following principles.
First choose the step-size $\alpha$ and the batch-size parameters (i.e., $B^0$ and $\eta_1$) that work well for SGD, then keep them the same in CSGD and LAG-S. Second, tune the control-size parameters (i.e., $\sigma^0, \eta_2$) to reach a considerable communication-saving with tolerable difference in the convergence with respect to the number of iterations.

\vspace{0.1cm}
{\bf Least squares.}
We first test on the least squares problem, given by
\begin{equation}
    f_m(\mathbf{x};\xi_m)=\frac12\|( \xi_m^{(1)})^T (\mathbf{x}-\mathbf{x}^*)+\xi_m^{(2)}\|^2,
~~~ \mathbf{x}\in\mathbb{R}^{10}.\label{obj:ls}
\end{equation}
Therein, entries of $ \xi_m^{(1)}\in\mathbb{R}^{10}$ are randomly chosen from the standard Gaussian,
entries of $\mathbf{x}^*$ are uniformly sampled from $[-2,2]$, and
$\xi_m^{(2)}\in\mathbb{R}$ is a Gaussian noise with distribution $\mathcal{N}(0,0.01^2)$.
All values are generated independently. The parameters are set as $\alpha=0.02$, $B^0=1$, $(1-\eta_1)^{-1}=1.1$, $\sigma^0=0.1$, and $1-\eta_2=0.91$, which guarantee that the condition $\eta_1>\eta_2$ in Theorem \ref{thm:sav} is satisfied.
From Figure \ref{fig:ls}, we observe that CSGD significantly saves the communication but converges to the same loss within a slightly more number of iterations.

We also use an intuitive explanation in Figure \ref{fig:ls_comm} to showcase the effectiveness of CSGD on censoring gradient uploads.
One blue stick refers to one gradient upload for the corresponding worker at that iteration.
The first $30$ iterations in this experiment adjust the initial variable to the point where newly calculated gradients become less informative, and after that communication events happen sparsely. Note that the uploads in Figure \ref{fig:ls_comm} are not as sparse as our theorems suggest --- no more than one communication event happens in $D$ consecutive iterations,
since we set $B^0=1$ instead of a sufficiently large number.
Nevertheless, the design of the batch-size and the control-size sparsifies the communication, and results in the significant reduction of communication cost after the first $300$ uploads as in Figure \ref{fig:ls}.
Besides, a well-designed control-size also plays an important role; without the control term, the curve of LAG-S highly overlaps with that of SGD, while CSGD outperforms the other two methods with the consideration of communication-efficiency.

\begin{figure}[h!]
\vspace{-0.2cm}
\centering
\includegraphics[height=5cm,width=7cm]{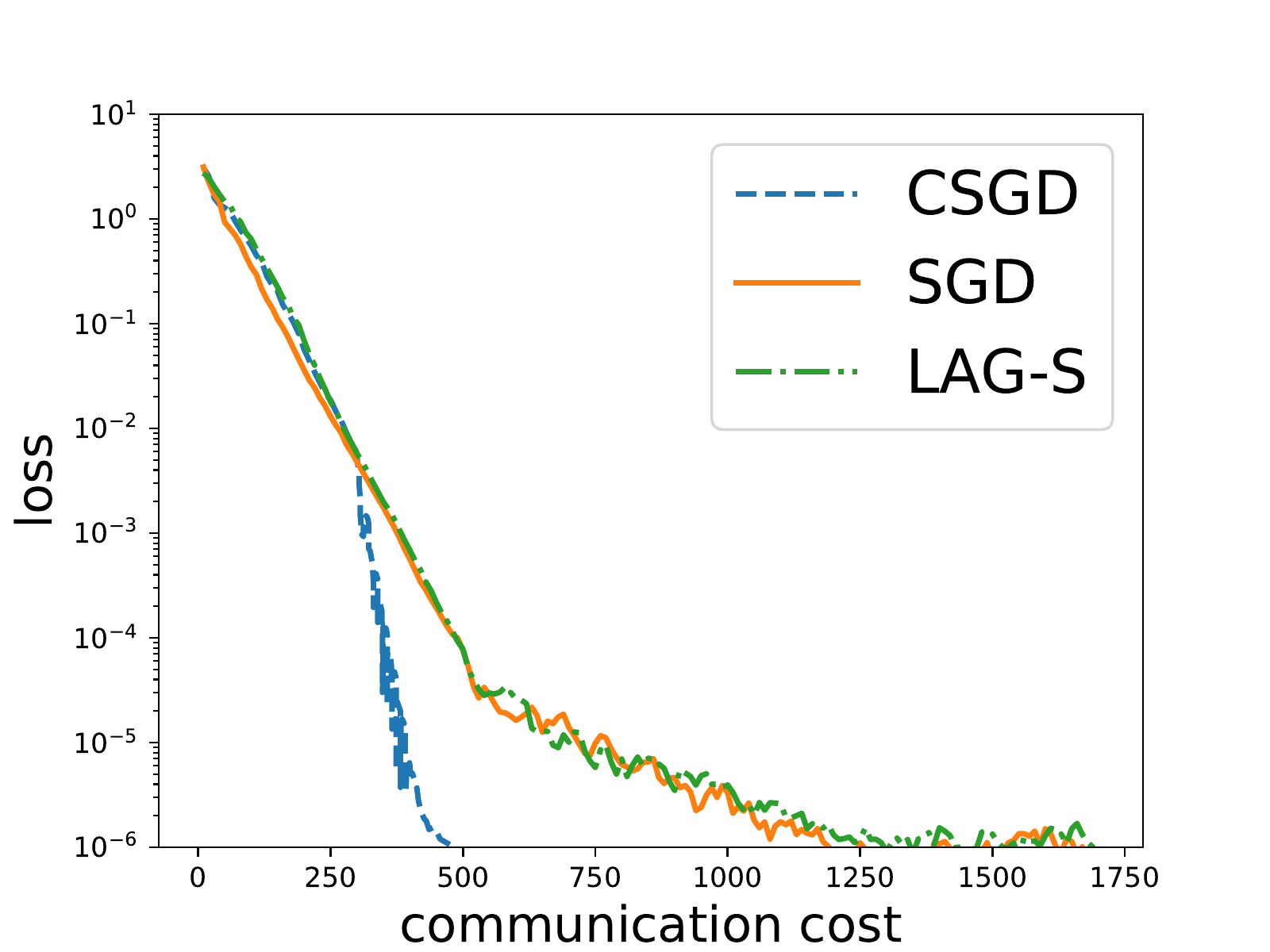}
\includegraphics[height=5cm,width=7cm]{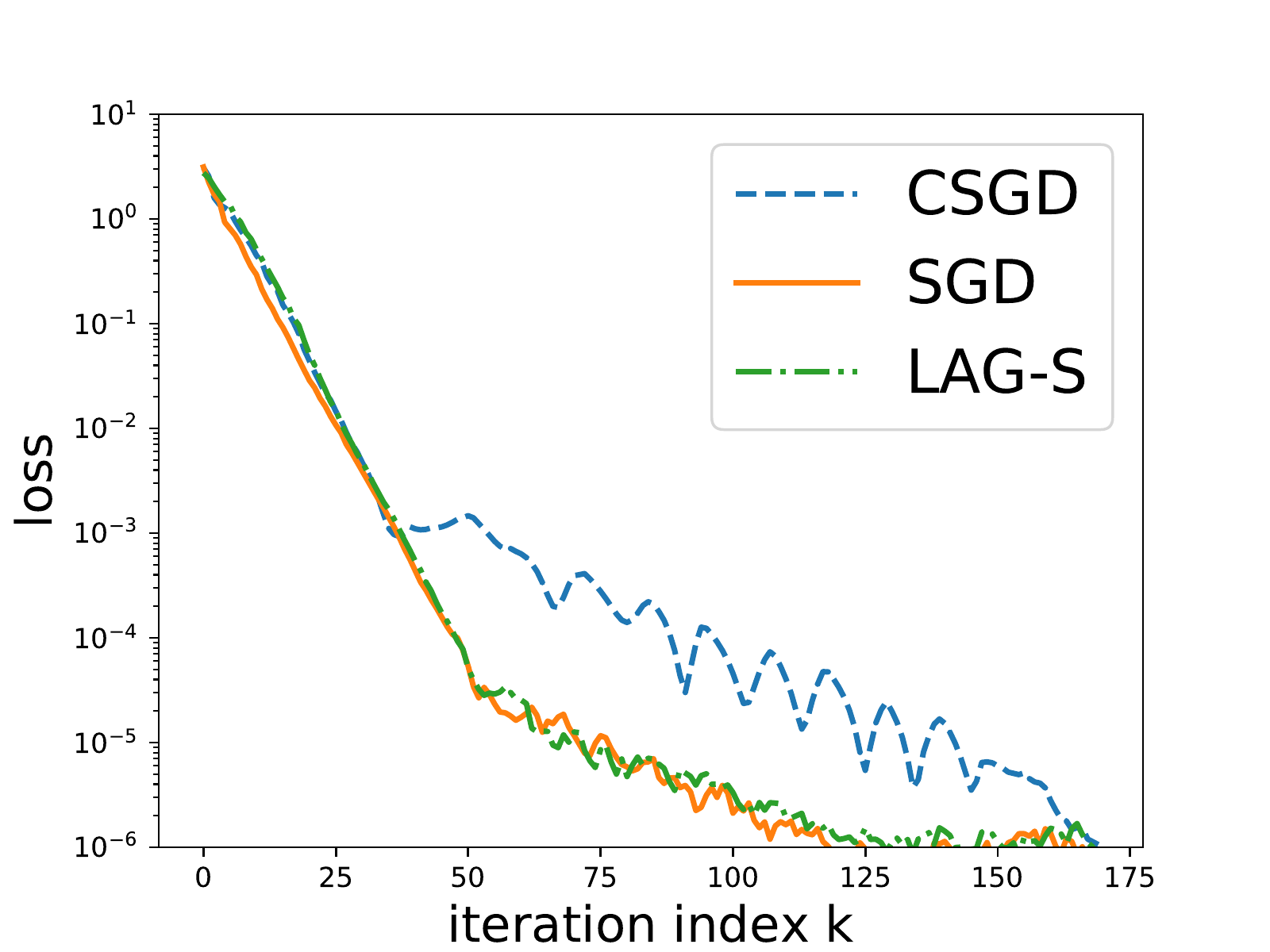}
\caption{Performance of three algorithms on the least squares problem.
Left: the loss versus communication cost.
Right: the loss versus number of iterations. }\label{fig:ls}
\vspace{-0.2cm}
\end{figure}

\begin{figure}[h!]
\centering
  \includegraphics[height=4.8cm]{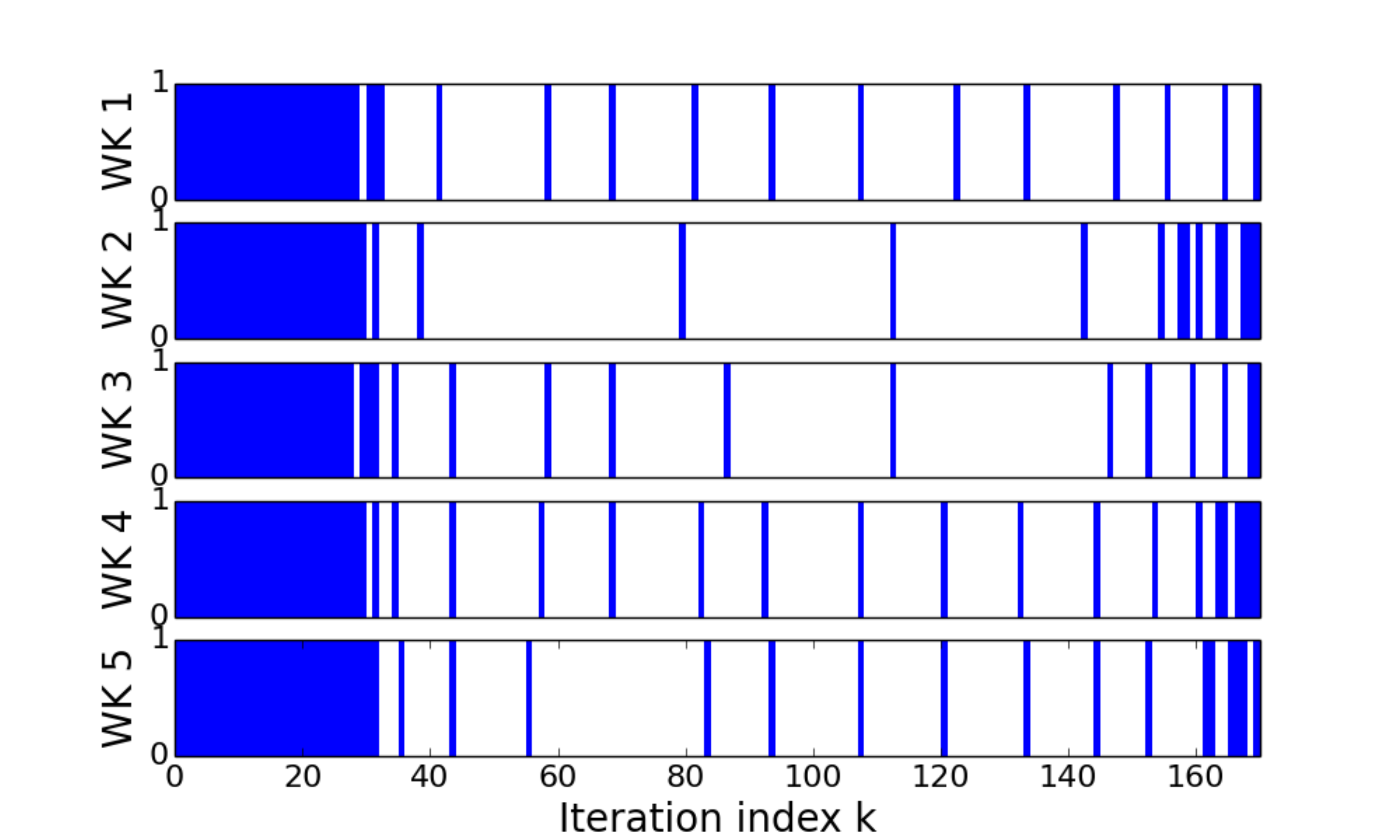}
  \caption{Uploading events of workers 1-5 over $170$ iterations in least squares. Each stick is an upload.}
\label{fig:ls_comm}
\vspace{-0.2cm}
\end{figure}

\vspace{0.1cm}
{\bf Softmax regression.}
Second, we conduct experiments on the MNIST dataset \cite{lecun1998mnist}, which has 60k training samples, and we use $2$-norm regularized
softmax regression with regularization parameter $0.0005$. The training samples are randomly and evenly assigned to the workers. The parameters are set as $\alpha=0.01$,  $B^0=1$, $(1-\eta_1)^{-1}=1.05$, $\sigma^0=15,1-\eta_2=0.96$.
From Figure \ref{fig:mnist}, the reduction of communication cost in CSGD can be easily observed, with a slightly slower convergence with respect to the number of iterations, which is consistent with the performance in the previous experiment.

\begin{figure}[h!]
\centering
\includegraphics[height=5cm,width=7cm]{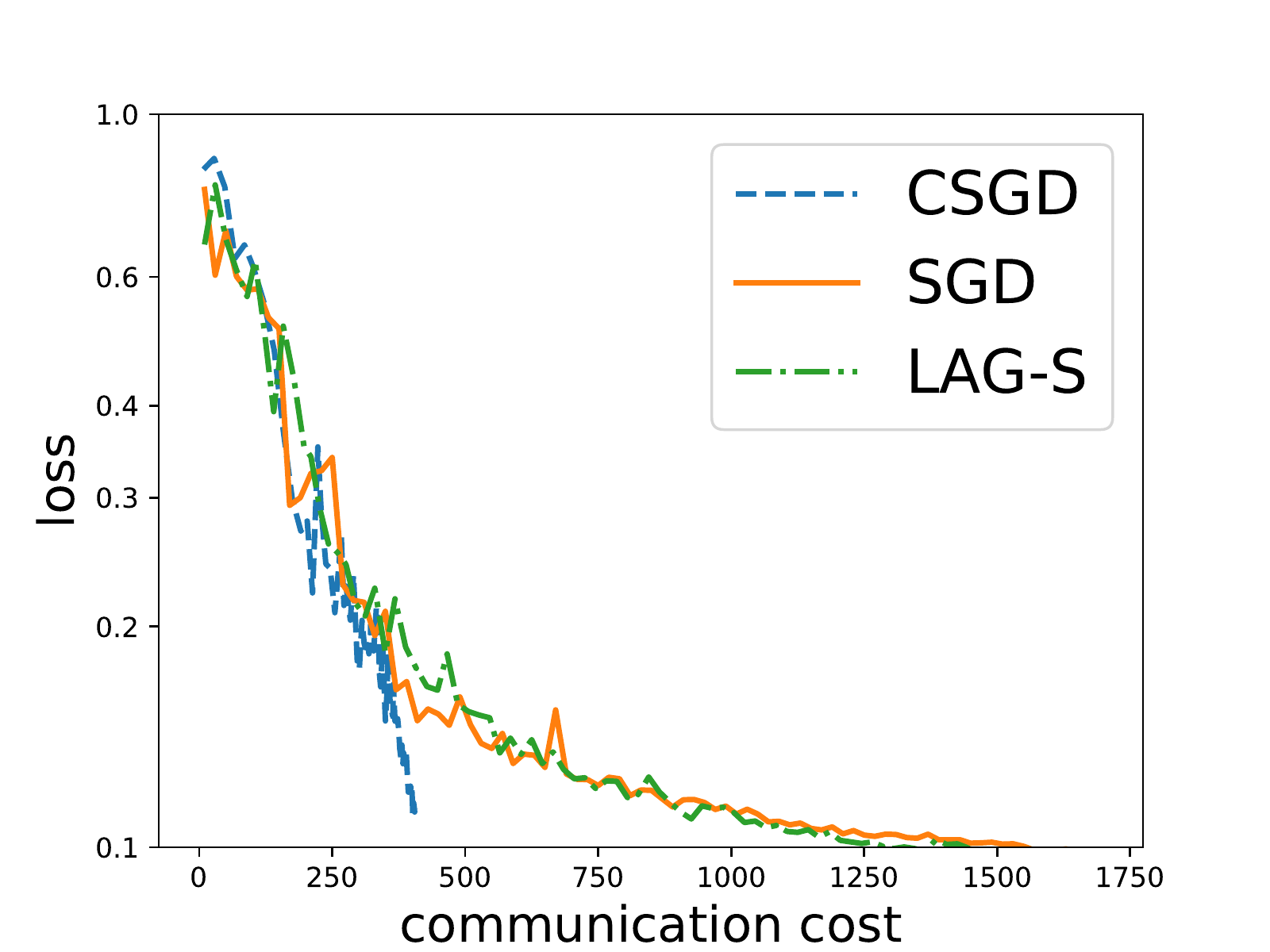}
\includegraphics[height=5cm,width=7cm]{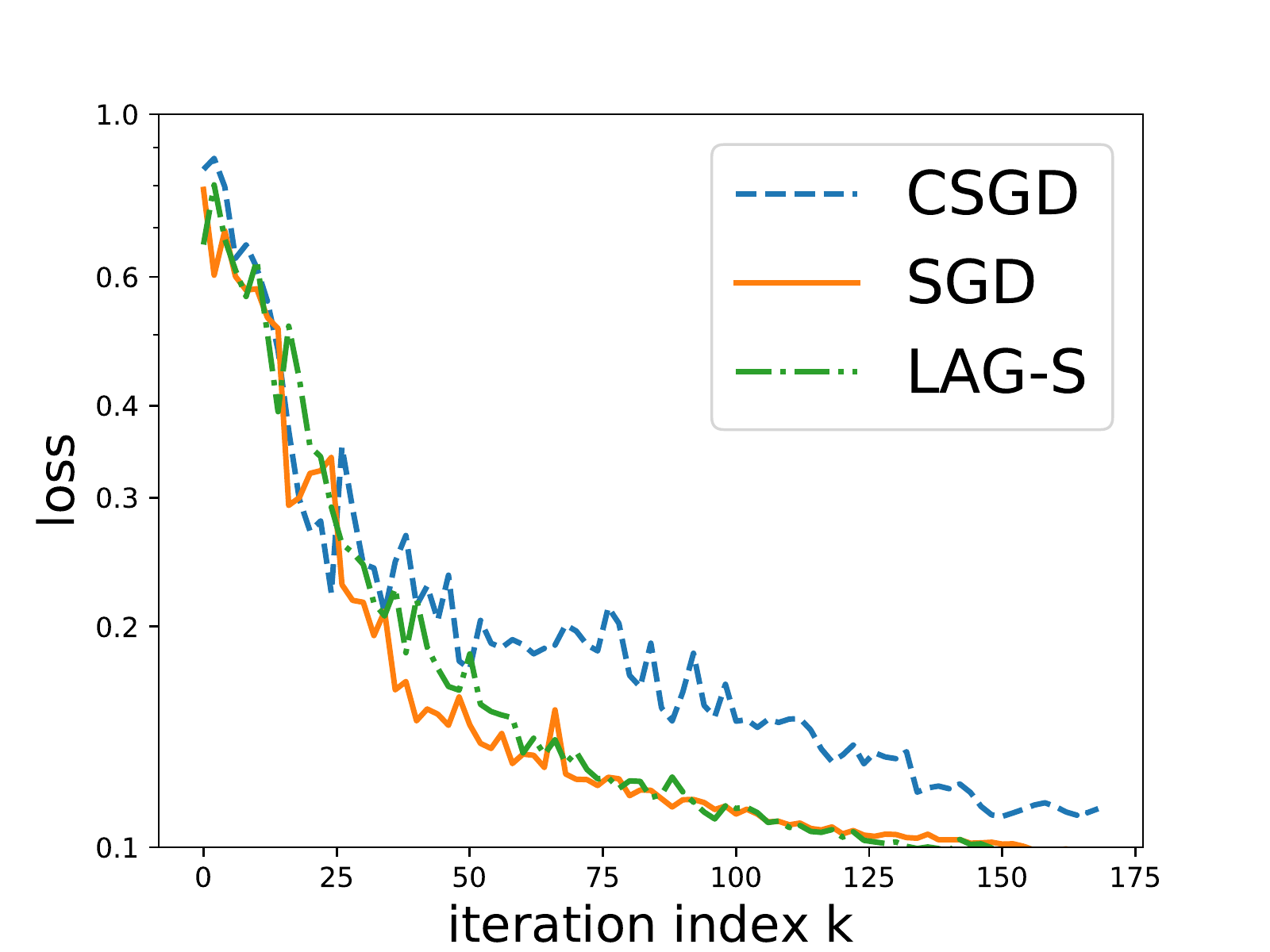}
 \vspace{-0.2cm}
\caption{Performance of three compared algorithms on the softmax regression problem.
Left: the loss versus the communication cost.
Right: the loss versus the number of iterations.  }
\label{fig:mnist}
\end{figure}

\begin{figure}[h!]
\vspace{-0.4cm}
\centering
\includegraphics[height=5cm,width=7cm]{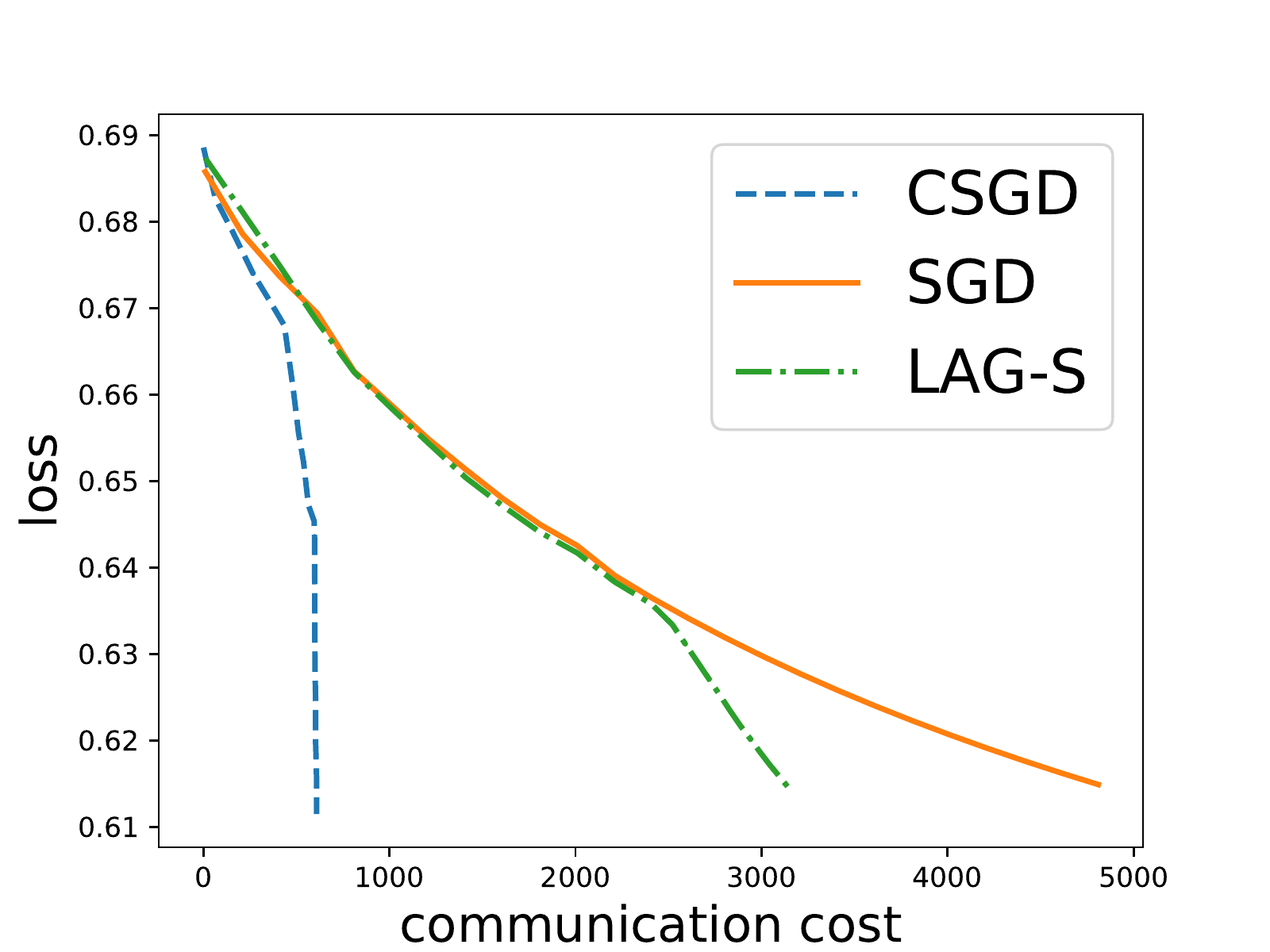}
\includegraphics[height=5cm,width=7cm]{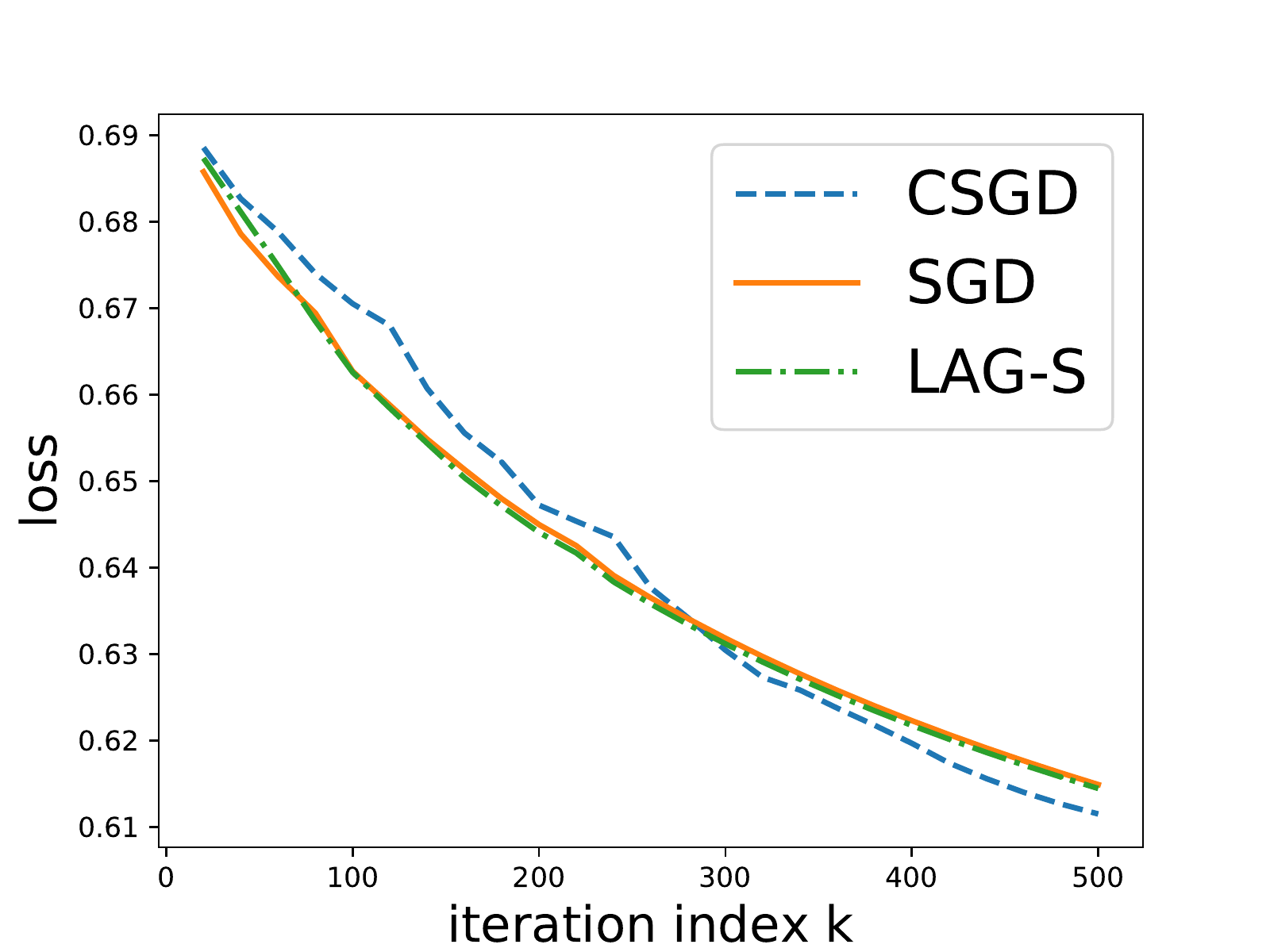}
 \vspace{-0.2cm}
\caption{Performance of the three compared algorithms on the logistic regression problem.
Left: the loss versus the communication cost.
Right: the loss versus the iterations.}
\vspace{-0.2cm}
\label{fig:Covertype}
\end{figure}

\vspace{0.1cm}
{\bf Logistic regression.}
Third, we conduct experiments on the Covertype dataset \cite{covtype} that has around $581$k training samples, and we use
$2$-norm regularized
logistic regression with regularization parameter $0.0005$. The training samples are randomly and evenly assigned to the workers, and the parameters are set as $\alpha=0.1$, $B^0=1$, $(1-\eta_1)^{-1}=1.05$, $\sigma^0=150$, $1-\eta_2=0.99$. Observed from Figure \ref{fig:Covertype}, though the three algorithms reach the same loss within similar numbers of iterations, CSGD and LAG-S require less communication cost. Compared to LAG-S that saves about $1/3$ communication after running for a long time, CSGD successfully saves communication at an initial stage, and achieves the overall significant communication savings.

\begin{figure}[h!]
	\vspace{-0.4cm}
	\centering
	\includegraphics[height=5cm,width=7cm]{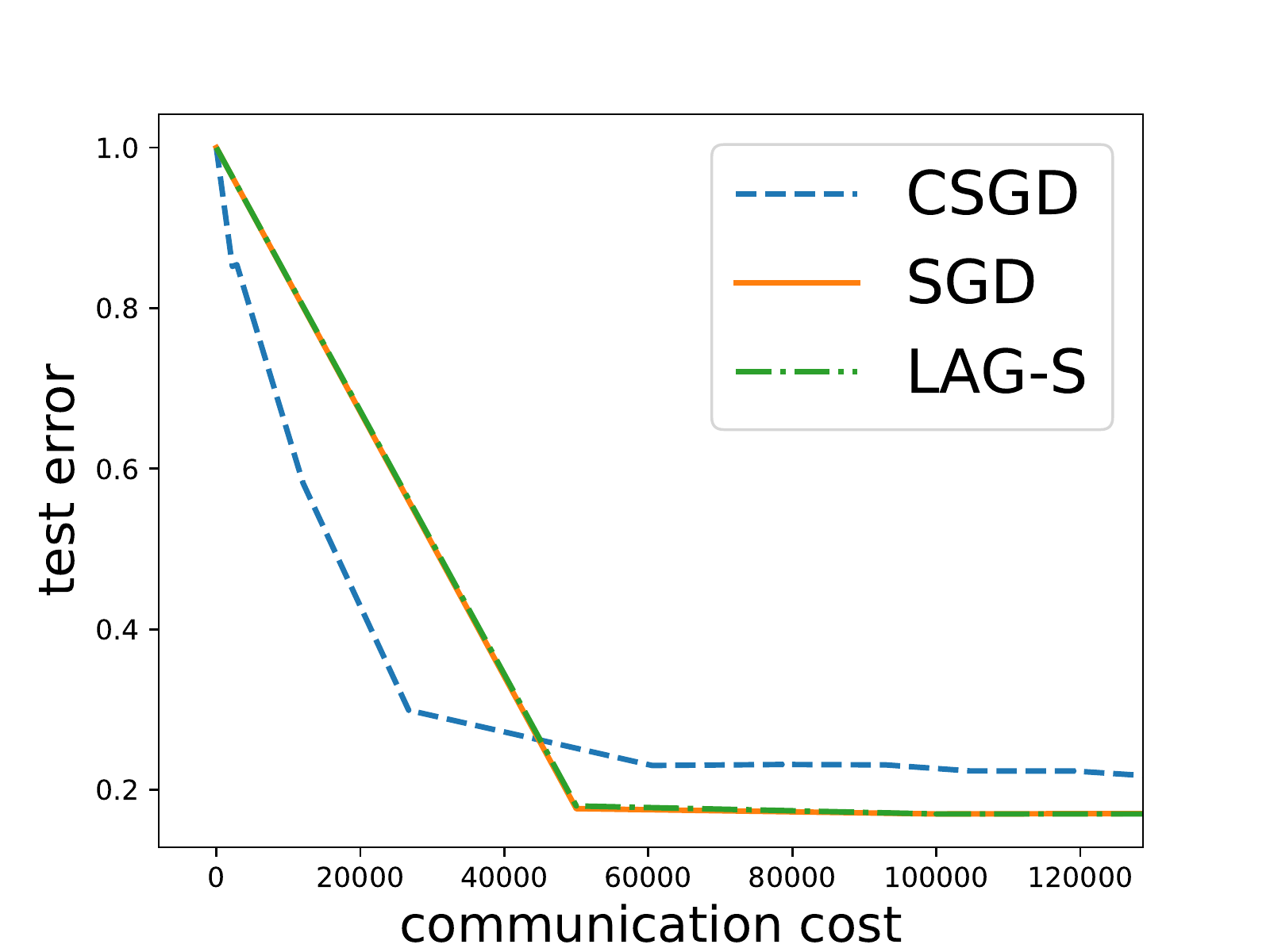}
	\includegraphics[height=5cm,width=7cm]{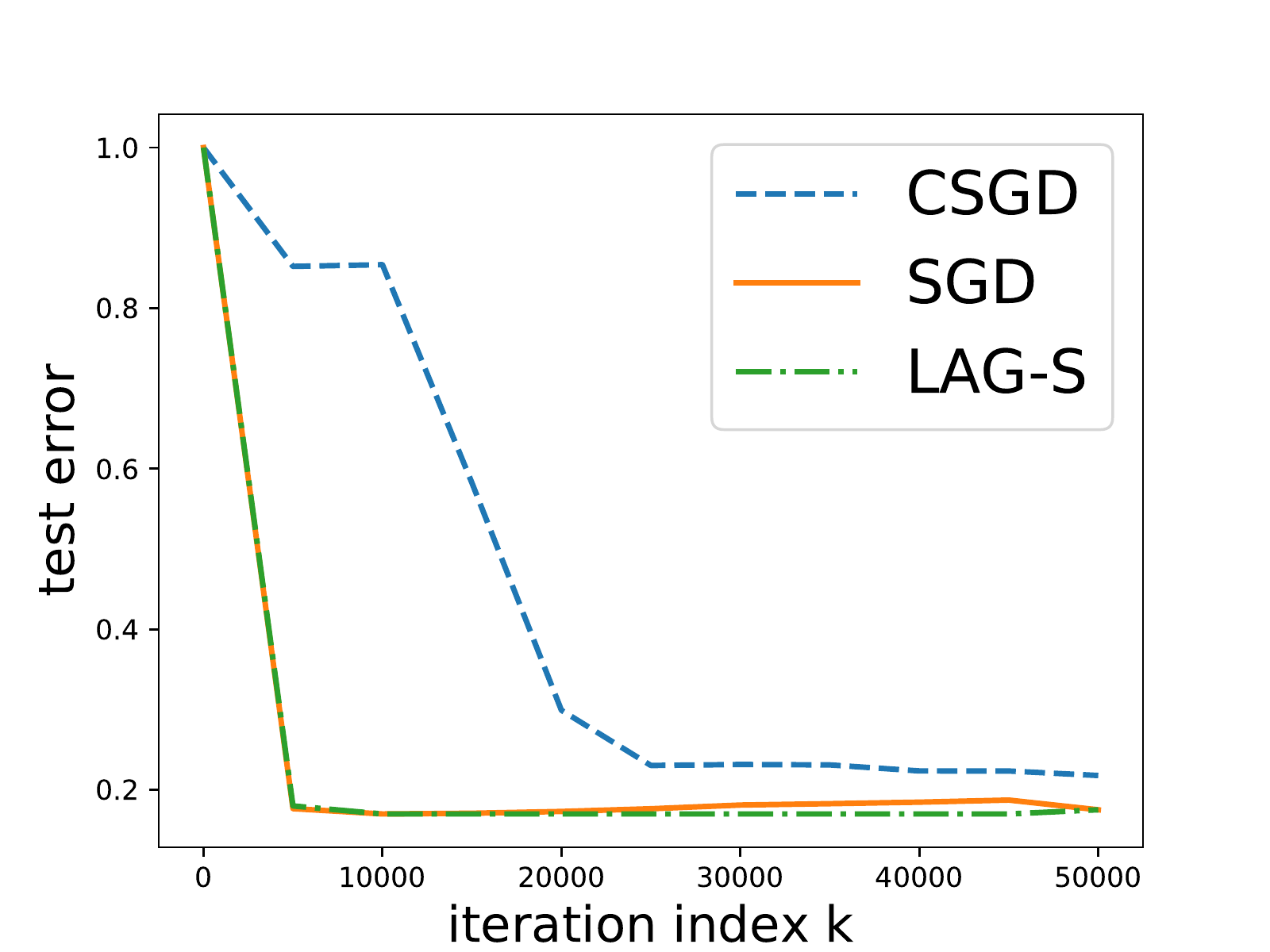}
	\vspace{-0.2cm}
	\caption{Performance of the three compared algorithms on the deep neural network training problem.
		Left: the test error versus the communication cost.
		Right: the test error versus the iterations.}
	\vspace{-0.2cm}
	\label{fig:cifar10}
\end{figure}

\vspace{0.1cm}
{\bf Deep neural network training.}
Finally, we train a ResNet-18 \cite{Kaiming2018} with regularization coefficient $0.0001$ on the CIFAR-10 dataset \cite{cifar10} that has $50$k samples for training and $10$k for testing. To avoid the batch-size and the control-size changing too rapidly, we replace the terms $B^0(1-\eta_1)^{-k}$ and $(1-\eta_2)^k$ in \eqref{eq:tauk-exp} by $B^0(1-\eta_1)^{-\lceil k/500 \rceil}$ and $(1-\eta_2)^{\lceil k/500 \rceil}$, respectively. The parameters are set as $\alpha=0.3$, $B^0=16$, $(1-\eta_1)^{-1}=1.05$, $\bar{B}=64$, $\sigma^0=1720$, and $1-\eta_2=0.96$. As illustrated by Figure \ref{fig:cifar10}, CSGD requires less communication cost to reach the same test accuracy. The curve of LAG-S almost overlaps with that of SGD, due to the lack of control-size that helps censoring the stochastic gradient noise.

\section{Conclusions and discussions}
We focused on the problem of communication-efficient distributed machine learning in this paper.
Targeting higher communication-efficiency,
we developed a new stochastic distributed optimization algorithm
abbreviated as CSGD.
By introducing a communication-censoring protocol,
CSGD significantly reduces the number of gradient uploads,
while it only sacrifices slightly the needed number of iterations.
It has been rigorously established that our proposed CSGD method
achieves the same order of convergence rate as the uncensored SGD,
while CSGD guarantees fewer number of gradient uploads
if a sufficient number of historic variables are utilized.
Numerical tests demonstrated the communication-saving merit of CSGD.




\appendices
\section{Supporting lemma}
\begin{lemma}[Quasi-martingale convergence]\label{lem:mar}
Let $(X_n)_{n\ge0}$ be a nonnegative sequence adapted to a filtration $(\mathcal{F}_n)_{n\ge0}$.
Denote $D_n=X_n-X_{n-1}$ for any $n\ge1$.
If $$\sum_{n=1}^\infty\mathbb{E}\big[D_n I_{\mathbb{E}[D_n|\mathcal{F}_{n-1}]>0}\big]
=\sum_{n=1}^\infty\mathbb{E}\big[\min\big\{\mathbb{E}[D_n|\mathcal{F}_{n-1}],0\big\}\big]<\infty,$$
where $I_{\mathbb{E}[D_n|\mathcal{F}_{n-1}]>0}$ is the indicator function of the event $\{\mathbb{E}[D_n|\mathcal{F}_{n-1}]>0\}$,
then there exists a random variable $X_\infty$, such that when $n\rightarrow\infty$,
$$X_n\overset{a.s.}\rightarrow X_\infty\ge0.$$
\end{lemma}

\begin{proof}
{\bf Step 1. Decomposition.}
We claim that $(X_n)_{n\ge0}$ can be decomposed into the sum of a submartingale and a supermartingale.
The construction is as follows.

Let $Y_{n-1}:=\mathbb{E}[D_n|\mathcal{F}_{n-1}]$, then $I_{Y_{n-1}>0}\in\mathcal{F}_{n-1}$.
Letting $D_n^+:=D_nI_{Y_{n-1}>0}$, we have that
$$\mathbb{E}[D_n^+|\mathcal{F}_{n-1}]=Y_{n-1}I_{Y_{n-1}>0}\ge0.$$
Similarly, we let $D_n^-:=D_nI_{Y_{n-1}\le0}.$
Therefore, if we set
\begin{align*}
&X_n^+=X_{n-1}^+ + D_n^+, \quad X_0^+=0,\\
\text{ and } &X_n^-=X_{n-1}^- +D_n^-,\quad X_0^-=X_0,
\end{align*}
then $X_n^+$, $X_n^-$ are sub- and super-martingale with respect to $\mathcal{F}_n$, respectively, and $X_n=X_n^++X_n^-$.

{\bf Step 2. Martingale convergence.}
From Theorem 5.2.8 in \cite{Durrett2010}, if $U_n$ is a submartingale with
$\sup \mathbb{E}\min\{U_n,0\}<\infty$,
then as $n\rightarrow\infty$, $U_n$ converges a.s. to an absolutely integrable limit $U$.
For here, notice that
\begin{align*}
&\mathbb{E}\min\{X_n^+,0\}\\=&\mathbb{E}\big[\mathbb{E}[\min\{X_n^+,0\}|\mathcal{F}_{n-1}]\big]\\
=&\mathbb{E}\big[\mathbb{E}[\min\{X_{n-1}^+ +Y_{n-1}I_{Y_{n-1}>0},0\}|\mathcal{F}_{n-1}]\big]\\
\le&\mathbb{E}\big[\mathbb{E}[\min\{X_{n-1}^+ ,0\} +\min\{Y_{n-1},0\}|\mathcal{F}_{n-1}]\big]\\
=&\mathbb{E}\big[\min\{X_{n-1}^+,0\}\big]+\mathbb{E}\big[\min\{\mathbb{E}[D_n|\mathcal{F}_{n-1}],0\}\big],
\end{align*}
which iteratively derives
\begin{align*}
&\mathbb{E}\big[\min\{X_n^+,0\}\big]\\\le&
\mathbb{E}\big[\min\{X_0^+,0\}\big]+
\sum_{n=1}^\infty\mathbb{E}\big[\min\{\mathbb{E}[D_n|\mathcal{F}_{n-1}],0\}\big]<\infty.
\end{align*}
Using the cited martingale convergence theorem, we have $X_n^+\overset{a.s.}\rightarrow X_\infty^+.$

On the other hand,
the submartingale $-X_n^-$ is no more than $X_n^+$, since $X_n^+ +X_n^-=X_n\ge0$.
Then we have $\sup\mathbb{E}\big[\min\{-X_n^-,0\}\big]\le\sup\mathbb{E}\big[\min\{X_n^+,0\}\big]<\infty$.
Again, the martingale convergence theorem shows that $X_n^-\overset{a.s.}\rightarrow X_\infty^-.$
Summing up those two sequences yields the desired convergence.
Further, the non-negativeness of $X_n$ provides that $X_\infty\ge0.$
\end{proof}

\section{Proof of Lemma \ref{lem:descent}}\label{app:CensorUpdate}
We first prove the CSGD part, and then the SGD part can be obtained with slight modifications.
Notice that
\begin{align}
\|\hat\nabla^k\|^2=&\|\sum_{m=1}^M\hat\nabla_m^k\|^2\notag\\
\le&(1+\epsilon)\|\sum_{m=1}^M\nabla_m^k\|^2+
(1+\frac1{\epsilon})\|\sum_{m=1}^M(\hat\nabla_m^k-\nabla_m^k)\|^2\notag\\
\le&(1+\epsilon)\|\sum_{m=1}^M\nabla_m^k\|^2+
(1+\frac1{\epsilon})M\sum_{m=1}^M\|\hat\nabla_m^k-\nabla_m^k\|^2\notag\\
\le&(1+\epsilon)\|\sum_{m=1}^M\nabla_m^k\|^2+(1+\frac1{\epsilon})M\sum_{m=1}^M\tau^k\notag\\
=&(1+\epsilon)\|\nabla^k\|^2+(1+\frac1{\epsilon})M^2\tau^k,\label{eq:nabla-hat}
\end{align}
where the first inequality comes from $\|y+z\|^2=\|y\|^2+\|z\|^2+2\langle y,z\rangle
\le (1+\epsilon)\|y\|^2+(1+\frac1{\epsilon})\|z\|^2$ for any $\epsilon>0$.
Secondly, the inequality $\langle y,z\rangle\le\frac\epsilon4\|y\|^2+\frac1{\epsilon}\|z\|^2$ gives that
\begin{align}
&\langle\nabla F(\mathbf{x}^{k-1}),\nabla  F(\mathbf{x}^{k-1})-\hat\nabla^k\rangle \label{eq:nabla-hat2} \\
=&\langle\nabla F(\mathbf{x}^{k-1}),\nabla  F(\mathbf{x}^{k-1})-\nabla^k+\nabla^k-\hat\nabla^k\rangle\notag\\
=&\langle\nabla F(\mathbf{x}^{k-1}),\nabla  F(\mathbf{x}^{k-1})-\nabla^k\rangle+\langle\nabla F(\mathbf{x}^{k-1}),\nabla^k-\hat\nabla^k\rangle\notag\\
\le&\frac\epsilon4\|\nabla F(\mathbf{x}^{k-1})\|^2+\frac1{\epsilon}\|\nabla F(\mathbf{x}^{k-1})-\nabla^k\|^2\notag\\
&+\frac\epsilon4\|\nabla F(\mathbf{x}^{k-1})\|^2+\frac1{\epsilon}\|\nabla^k-\hat \nabla^k\|^2\notag\\
=&\frac\epsilon2\|\nabla F(\mathbf{x}^{k-1})\|^2+\frac1{\epsilon}\|\nabla F(\mathbf{x}^{k-1})-\nabla^k\|^2+\frac1{\epsilon}M^2\tau^k. \nonumber
\end{align}
From \eqref{CSGD} and the Lipschitz continuity of $\nabla F$, we have
\begin{align*}
&F(\mathbf{x}^k)-F(\mathbf{x}^{k-1})\\
=&F(\mathbf{x}^{k-1}-\alpha\hat\nabla^k)-F(\mathbf{x}^{k-1})\\
\le& -\alpha \langle\nabla F(\mathbf{x}^{k-1}),\hat\nabla^k\rangle
+\frac L2\alpha^2\|\hat\nabla^k\|^2\\
=& -\alpha \| \nabla F(\mathbf{x}^{k-1})\|^2+
\alpha \langle\nabla F(\mathbf{x}^{k-1}),\nabla  F(\mathbf{x}^{k-1})-\hat\nabla^k\rangle \nonumber \\
& +\frac{L}2\alpha^2\|\hat\nabla^k\|^2\\
\le&-\alpha\bigg(1 -\frac\epsilon2 -(1+\epsilon)\frac{L}2\alpha\bigg)\| \nabla F(\mathbf{x}^{k-1})\|^2 \nonumber \\
& +\alpha M^2\bigg(\frac{1}{\epsilon}+(1+\frac1{\epsilon})\frac{L}2\alpha \bigg)\tau^k
+\frac{\alpha }{\epsilon}\|\nabla F(\mathbf{x}^{k-1})-\nabla^k\|^2,
\end{align*}
where the last inequality uses \eqref{eq:nabla-hat} and \eqref{eq:nabla-hat2}.

For SGD with update rule \eqref{SGD}, \eqref{lem1} can be derived by replacing \eqref{eq:nabla-hat2} with
\begin{align*}
&\langle\nabla F(\bar{\mathbf{x}}^{k-1}),\nabla  F(\bar{\mathbf{x}}^{k-1})-\bar\nabla^k\rangle\\
\le&\frac\epsilon2\|\nabla F(\bar{\mathbf{x}}^{k-1})\|^2+\frac1{2\epsilon}\|\nabla F(\bar{\mathbf{x}}^{k-1})-\bar\nabla^k\|^2.
\end{align*}

\section{Proofs of Theorems \ref{thm:as} and \ref{thm:rate}}\label{app2}
\begin{lemma}[Lyapunov descent]\label{lem:lya}
For $k\ge1$, denote $\mathcal{F}_{k-1}=\sigma(\{\mathbf{x}^j,j\le k-1\})$ and
$\bar{\mathcal{F}}_{k-1}=\sigma(\{\bar{\mathbf{x}}^j,j\le k-1\})$ as the natural sigma-fields generated by the random vectors before iteration $k-1$ in CSGD and SGD, respectively.
Under Assumptions \ref{ass:global} and \ref{ass:convex},
if we choose $\beta_d=\frac{D+1-d}{9D}\alpha,w=\frac1{60}$,
step-size $\alpha \le\min\{\frac{3}{2D\mu},\frac1{3L}\}$,
then the Lyapunov functions satisfy
\begin{align}\label{eq:lya}
\mathbb{E}\big[V^k-V^{k-1}|\mathcal{F}_{k-1}\big]\le&
-\rho V^{k-1}+R^k,\\
\mathbb{E}\big[\bar V^k-\bar V^{k-1}|\bar{\mathcal{F}}_{k-1}\big]\le&
-\rho  \bar V^{k-1}+\bar{R}^k,\label{eq:lya-sgd}
\end{align}
where $\rho=\frac13\mu\alpha$, $R^k=\frac{10\alpha\sigma^k}3+\frac{7\alpha G^2}{3B^k}$, and $\bar{R}^k=\frac{7\alpha G^2}{3B^k}$.
\end{lemma}

\begin{proof}
By convention, define $\beta_{D+1}=0$.
From the definition of $V^k$ in \eqref{Lyapunov} and the inequality in \eqref{lem1},
\begin{align}
&V^k-V^{k-1}\notag\\
=&F(\mathbf{x}^k)-F(\mathbf{x}^{k-1})+\beta_1\|\hat\nabla^{k}\|^2-\sum_{d=1}^{D}(\beta_d-\beta_{d+1})\|\hat\nabla^{k-d}\|^2\nonumber\\
\le&\Delta^k+\beta_1\|\hat{\nabla}^k\|^2-\sum_{d=1}^{D}(\beta_d-\beta_{d+1})\|\hat\nabla^{k-d}\|^2.\label{eq:Vk-Vk-1_bound}
\end{align}
Plugging in the definition of $\Delta^k$, and taking conditional expectation on $\mathcal{F}_{k-1}$ and $\epsilon=\bar\epsilon=\frac12$ hereafter,
we have
\begin{align}
&\mathbb{E}\big[\Delta^k|\mathcal{F}_{k-1}\big]\notag\\=&-\frac34\alpha(1 -L\alpha)\| \nabla F(\mathbf{x}^{k-1})\|^2
+\alpha M^2\bigg(2+\frac32L\alpha \bigg)\tau^k\notag\\&+2\alpha\mathbb{E}\|\nabla F(\mathbf{x}^{k-1})-\nabla^k\|^2\notag\\
\le&-\frac12\alpha\| \nabla F(\mathbf{x}^{k-1})\|^2+3\alpha M^2\tau^k+\frac{2\alpha G^2}{B^k},
\end{align}
since $L\alpha\le\frac13$ and
\begin{align}
&\mathbb{E}\|\nabla F(\mathbf{x}^{k-1})-\nabla^k\|^2\notag\\=&\frac1{(B^k)^2}\mathbb{E}\bigg\|\sum_{b=1}^{B^k}\big(\nabla F(\mathbf{x}^{k-1})-\nabla f(\mathbf{x}^{k-1};\xi^{k,b})\big)\bigg\|^2\notag\\
\le&\frac1{B^k}G^2.\label{eq:G2}
\end{align}
Further, \eqref{eq:G2} and $\|\hat{\nabla}^k\|^2\le3( \|\hat{\nabla}^k-\nabla^k\|^2+\|\nabla^k-\nabla F(\mathbf{x}^{k-1})\|^2+\|\nabla F(\mathbf{x}^{k-1})\|^2 )$ give that
\begin{equation*}
\mathbb{E}\big[\|\hat{\nabla}^k\|^2|\mathcal{F}_{k-1}\big]\le3\left( M^2\tau^k+\frac{G^2}{B^k}+\|\nabla F(\mathbf{x}^{k-1})\|^2 \right).
\end{equation*}
Thus, conditioned on $\mathcal{F}_{k-1}$, \eqref{eq:Vk-Vk-1_bound} becomes
\begin{align}
&\mathbb{E}\big[V^k-V^{k-1}|\mathcal{F}_{k-1}\big]\notag\\
\le&-\frac12\alpha\| \nabla F(\mathbf{x}^{k-1})\|^2+3\alpha M^2\tau^k+\frac{2\alpha G^2}{B^k}\notag\\
&+3\beta_1\left( M^2\tau^k+\frac{G^2}{B^k}+\|\nabla F(\mathbf{x}^{k-1})\|^2 \right)\notag\\
&-\sum_{d=1}^{D}(\beta_d-\beta_{d+1})\|\hat\nabla^{k-d}\|^2\notag\\
\overset{\eqref{threshold}}=&-\frac{\alpha - 6 \beta}{2}\| \nabla F(\mathbf{x}^{k-1})\|^2+3(\alpha+\beta_1)\sigma^k+(2\alpha+3\beta_1)\frac{G^2}{B^k}\notag\\
&-\sum_{d=1}^{D}(\beta_d-\beta_{d+1}-3(\alpha+\beta_1)\frac{w}D)\|\hat\nabla^{k-d}\|^2\notag\\
\le&-\frac16\alpha\| \nabla F(\mathbf{x}^{k-1})\|^2-\frac13\mu\alpha\sum_{d=1}^{D}\beta_d\|\hat\nabla^{k-d}\|^2\notag\\&+\frac{10}3\alpha\sigma^k+\frac73\alpha\frac{G^2}{B^k},\label{eq:Vbound1}
\end{align}
since $\beta_d=\frac{D+1-d}{9D}\alpha$, $w=\frac1{60}$, and
$$\beta_d-\beta_{d+1}-3(\alpha+\beta_1)\frac{w}D=\frac1{2D}\beta_1\ge\frac1{2D}\beta_d\ge\frac13\mu\alpha\beta_d$$
from $\frac13\mu\alpha\le\frac1{2D}$.
From Assumption \ref{ass:convex} that
$$\|\nabla F(\mathbf{x}^{k-1})\|^2\ge2\mu\big(F(\mathbf{x}^{k-1})-F^*\big),$$
and the notation $R^k=\frac{10}3\alpha\sigma^k+\frac73\alpha\frac{G^2}{B^k}$,
we obtain
\begin{align}
&\mathbb{E}\big[V^k-V^{k-1}|\mathcal{F}_{k-1}\big]\notag\\
\le&-\frac13\mu\alpha( F( \mathbf{x}^{k-1})-F^*)-\frac13\mu\alpha\sum_{d=1}^{D}\beta_d\|\hat\nabla^{k-d}\|^2+R^k\notag\\
\le&-\frac13\mu\alpha V^{k-1}+R^k.\label{eq:Vbound3}
\end{align}
Further denoting $\rho=\frac13\mu\alpha$ in \eqref{eq:Vbound3} leads to \eqref{eq:lya}.

Analogously, for the SGD algorithm (c.f. \eqref{SGD}),
\begin{align*}
&\mathbb{E}\big[\bar V^k-\bar V^{k-1}|\bar{\mathcal{F}}_{k-1}\big]\notag\\
\le&-\frac16\alpha\| \nabla F(\bar{\mathbf{x}}^{k-1})\|^2-\sum_{d=1}^{D}(\beta_d-\beta_{d+1})\|\bar\nabla^{k-d}\|^2+\frac73\alpha\frac{G^2}{B^k}\\
\le&-\frac13\mu\alpha \bar{V}^{k-1}+\frac73\alpha\frac{G^2}{B^k},
\end{align*}
since $\beta_d-\beta_{d+1}=\frac1D\beta_1\ge\frac1D \beta_d\ge\frac13\mu\alpha\beta_d$.
Now denoting $\bar{R}^k=\frac73\alpha\frac{G^2}{B^k}$ completes the proof of \eqref{eq:lya-sgd}.
\end{proof}

Now we can start our proofs of Theorems \ref{thm:as} and \ref{thm:rate}.
\begin{proof}
{\bf Step 1. A.s. convergence.}
From Lemma \ref{lem:lya} and the non-negativeness of Lyapunov functions,
we have
\begin{align*}
&\sum_{k=1}^\infty\mathbb{E}\big[\mathbb{E}[V^k-V^{k-1}|\mathcal{F}_{k-1}]\vee0\big]\\
\le&\sum_{k=1}^\infty\left(-\rho  V^{k-1}+R^k\right)\vee0
\le\sum_{k=1}^\infty R^k<\infty,
\end{align*}
where $R^k$ is summable since both $\sigma^k$ and $\frac1{B^k}$ is summable.
Therefore, from Lemma \ref{lem:mar}, there exists a random variable $V^\infty$ such that
$V^k\overset{a.s.}\rightarrow V^\infty\ge0$ as $k\rightarrow\infty$.

To conclude that $V^\infty\overset{a.s.}=0$, we assume $V^k\rightarrow V^\infty\ge e>0$
on some set $A$ in the probability space with $P(A)>0$.
For any $\omega\in A$, there exists an integer $k_0=k_0(\omega)$,
such that $V^{k-1}\ge\frac{e}2$ for all $k> k_0$.
Then from Lemma \ref{lem:lya},
$$\mathbb{E}[V^k-V^{k-1}|\mathcal{F}_{k-1}]\le
-\mu\alpha \frac{e}2 +R^k.$$
Iteratively using this fact, we obtain
\begin{align*}
\mathbb{E}[V^k-V^{k_0}|\mathcal{F}_{k_0}]\le&-\frac{\mu{e}}2\sum_{j=k_0+1}^k\alpha
+\sum_{j=k_0+1}^kR^k,
\end{align*}
which goes to $-\infty$ as $k\rightarrow\infty$.
Therefore, when $k$ is sufficiently large, $\mathbb{E}[V^k|\mathcal{F}_{k_0}]\le0$ on a set $A$ with positive probability, which is a contradiction.
In summary, $V^k\overset{a.s.}\rightarrow0.$

{\bf Step 2. Convergence rates of $V^k$ and $\bar{V}^k$.}
By conditioning on $\mathcal{F}_{k-1}$ first and then conditioning on $\mathcal{F}_0$, Lemma \ref{lem:lya} gives an important inequality
\begin{equation}
\mathbb{E}\big[V^k|\mathcal{F}_0\big]\le\big(1-\rho\big) \mathbb{E}\big[V^{k-1}|\mathcal{F}_0\big]+R^k,\label{eq:Vk|F0}
\end{equation}
which holds for all $k\ge1$.
Iteratively using \eqref{eq:Vk|F0} yields
\begin{align}
&\mathbb{E}\big[V^k|\mathcal{F}_0\big] \label{eq:Vk|F0-csgd}\\
\le&(1-\rho)^k V^0+\sum_{j=1}^k R^j(1-\rho)^{k-j}\nonumber\\
\overset{\eqref{stepsize}}\le&(1-\rho)^k \left(V^0+\frac{\alpha(1-\rho)}3(\frac{10\sigma^0}{\eta_2-\rho}+\frac{7G^2}{B^0(\eta_1-\rho)})\right).\nonumber
\end{align}
Similar result holds for SGD that
\begin{align}
&\mathbb{E}\big[\bar{V}^k|\mathcal{F}_0\big]\notag\\
\le&(1-\rho)^k V^0+\sum_{j=1}^k \bar{R}^j(1-\rho)^{k-j}\notag\\
\overset{\eqref{stepsize}}\le&(1-\rho)^k \left(V^0+\frac{7\alpha(1-\rho)G^2}{3B^0(\eta_1-\rho)}\right).\label{eq:Vk|F0-sgd}
\end{align}
\eqref{eq:Vk|F0-csgd} and \eqref{eq:Vk|F0-sgd} are exactly \eqref{Vbound}.
\end{proof}

\subsection{Proof of Theorem \ref{thm:sav}}
\begin{proof}
From the Markov inequality and \eqref{eq:G2},
\begin{align}
&\mathbb{P}(\|\nabla_m^k-\nabla F_m(\mathbf{x}^{k-1})\|^2>t^k)\notag\\\le&\frac1{t^k}\mathbb{E}\|\nabla_m^k-\nabla F_m(\mathbf{x}^{k-1})\|^2\le\frac1{t^kB^k}G_m^2.\label{markov}
\end{align}
We mainly focus on sample paths where $$\|\nabla_m^k-\nabla F_m(\mathbf{x}^{k-1})\|^2\le t^k$$ holds for all $k$ and $m$ with $t^k=\frac{\sigma^{k+D}}{6M^2}$.
Since
\begin{align}
&\sum_{k=1}^\infty\sum_{m=1}^M \frac1{t^kB^k}G_m^2\notag\\
\le&\sum_{k=1}^\infty\frac{6M^2\sum_{m=1}^MG_m^2}{\sigma^0(1-\eta_2)^{k+D}B^0(1-\eta_1)^{-k}}=\delta,\label{eq:delta}
\end{align}
such a sample path appears with probability at least $1-\delta$.
Suppose that at iteration $k$
when the worker $m$ decides to upload its latest gradient $\nabla_m^k$,
the most recent iteration that it did communicate with the server is $k-d'$; that is, $\hat\nabla_m^{k-1}=\nabla_m^{k-d'}$.

If $1\le d'\le D$, we next prove that it contradicts with the censoring threshold.
On the one hand, $\|\nabla_m^k-\nabla_m^{k-d'}\|^2>\tau^k$ since communication happens at iteration $k$.
On the other hand, we have
\begin{align}
&\|\nabla_m^k-\nabla_m^{k-d'}\|^2\notag\\
\le&3\big(\|\nabla_m^k\!-\!\nabla F_m(\mathbf{x}^{k-1})\|^2+\|\nabla F_m(\mathbf{x}^{k-d'-1})\!-\!\nabla_m^{k-d'}\|^2\notag\\&+\|\nabla F_m(\mathbf{x}^{k-1})\!-\!\nabla F_m(\mathbf{x}^{k-d'-1})\|^2\big)\notag\\
\le&\frac{\sigma^{k+D}}{2M^2}+\frac{\sigma^{k-d'+D}}{2M^2}+3d'\sum_{d=1}^{d'}\|\nabla F_m(\mathbf{x}^{k-d})-\nabla F_m(\mathbf{x}^{k-d-1})\|^2\notag\\
\overset{(a)}\le&\frac{\sigma^{k}}{2M^2}+\frac{\sigma^{k}}{2M^2}+3d'L_m\sum_{d=1}^{d'}\|\mathbf{x}^{k-d}-\mathbf{x}^{k-d-1}\|^2\notag\\
=&\frac{\sigma^{k}}{M^2}+3d'L_m\alpha^2\sum_{d=1}^{d'}\|\hat\nabla^{k-d}\|^2\notag\\
\overset{(b)}\le&\frac{ \sigma^k }{M^2}+\frac1{M^2}\frac{w}D\sum_{d=1}^{d'}\|\hat\nabla^{k-d}\|^2\le\tau^k, \label{eq:contra}
\end{align}
where $(a)$ comes from $ \sigma^{k+D}\le \sigma^{k-d'+D}\le\sigma^{k} $ and the Lipschitz continuity of $\nabla F_m$ in Assumption \ref{ass:local}, and $(b)$ comes from $\alpha \le\frac1{6\sqrt{5 L_m}MD}$.
The contradiction results in the conclusion that
at most one communication happens in $D$ consecutive iterations for every worker.
Consequently, the number of communications is at most $\lceil\frac{K-1}{D}\rceil+1$
after $K$ iterations for worker $m$.

To reach the same accuracy after running CSGD for $K$ iterations,
from Theorem \ref{thm:rate}
the iteration complexity of SGD is lower-bounded that $$\bar{K}\ge K-1+\big(\log(1-\rho)\big)^{-1}\log(\frac{C_{\rm SGD}}{C_{\rm CSGD}}).$$
Then in order to have less communication complexity, a sufficient condition is that
\begin{equation}
\lceil\frac{K-1}{D}\rceil+1\le \bar{K}.\label{eq:suff}
\end{equation}
Since
$\lceil\frac{K-1}{D}\rceil+1\le\frac{K}{2}+2$ for $D\ge2$,
\eqref{eq:suff} holds with
$$K\ge 6,$$
which completes the proof.
\end{proof}

\section{Proof of Theorem \ref{thm:nonconvex}}
\begin{proof}
{\bf Step 1. A.s. convergence (not necessarily converging to zero).}
Without Assumption \ref{ass:convex}, it is unable to derive Lemma \ref{lem:lya}, but \eqref{eq:Vbound1} in its proof still holds, from which we know that
\begin{equation}
\mathbb{E}\big[V^k-V^{k-1}|\mathcal{F}_{k-1}\big]\le R^k,
\end{equation}
where $R^k$ is summable. So applying Lemma \ref{lem:mar} gives the a.s. convergence that
$V^k\overset{a.s.}\rightarrow V^\infty\ge0.$
Hereafter we focus on the case
\begin{equation}\label{eq:path}
\begin{cases}
V^k\rightarrow V^\infty\ge0,\\
\|\nabla_m^k-\nabla F_m(\mathbf{x}^{k-1})\|^2\le \frac{\sigma^{k+D}}{6M^2},\quad \forall k,m,
\end{cases}
\end{equation}
which happens with probability at least $1-\delta$, following the same calculation as in \eqref{eq:delta}.

{\bf Step 2. Bounds of the Lyapunov differences.}
Without Assumption \ref{ass:convex}, inequality similar to \eqref{eq:Vbound1} still holds.
The only difference comes from using
\begin{align*}
&\|\nabla^k-\nabla F(\mathbf{x}^{k-1})\|^2\notag\\\le& M^2\sum_{m=1}^M\|\nabla_m^k-\nabla F_m(\mathbf{x}^{k-1})\|^2\le \frac{\sigma^{k+D}}{6}
\end{align*}
 instead of taking expectations that
$$\mathbb{E}\|\nabla^k-\nabla F(\mathbf{x}^{k-1})\|^2\le\frac{G^2}{B^k}.$$
Thus, replacing $\frac{G^2}{B^k}$ in \eqref{eq:Vbound1} by $\frac{\sigma^{k+D}}{6}$ yields
\begin{align}
&V^k-V^{k-1}\notag\\
\le&-(\frac12\alpha-3\beta_1)\| \nabla F(\mathbf{x}^{k-1})\|^2\notag\\
&-\sum_{d=1}^{D}(\beta_d-\beta_{d+1}-3(\alpha+\beta_1)\frac{w}D)\|\hat\nabla^{k-d}\|^2\notag\\
&+3(\alpha+\beta_1)\sigma^k+(2\alpha+3\beta_1)\frac{\sigma^{k+D}}{6}\notag\\
\le&-\frac16\alpha\| \nabla F(\mathbf{x}^{k-1})\|^2-\sum_{d=1}^{D}\frac{\alpha}{18D}\|\hat\nabla^{k-d}\|^2\notag\\&+\frac{10}3\alpha\sigma^k+\frac7{18}\alpha\sigma^{k+D},\label{eq:dif}
\end{align}
where $S^k=\frac{10}3\alpha\sigma^k+\frac7{18}\alpha\sigma^{k+D}$ is summable with $\sum_{k=1}^\infty S^k\le\frac{67}{18}\alpha\sigma^0\sum_{k=1}^\infty (1-\eta_2)^k=\frac{67\alpha(1-\eta_2)\sigma^0}{18\eta_2}$.
Then, summing \eqref{eq:dif} from $k=1$ to $k=K$ gives
\begin{align*}
&V^K-V^0\notag\\\le&-\frac16\alpha\sum_{k=1}^K\| \nabla F(\mathbf{x}^{k-1})\|^2
-\frac\alpha{18D}\sum_{d=1}^{D}\sum_{k=1}^K\|\hat\nabla^{k-d}\|^2\notag\\&+\frac{67\alpha(1-\eta_2)\sigma^0}{18\eta_2},
\end{align*}
which implies
\begin{align}
\sum_{k=1}^K\| \nabla F(\mathbf{x}^{k-1})\|^2\le&\frac6{\alpha} N_{CSGD},\label{eq:sum-nablaF}\\
\sum_{k=1}^K\| \hat\nabla^{k}\|^2\le&\frac{18}\alpha N_{CSGD},
\end{align}
with $N_{CSGD}=V^0+\frac{67\alpha(1-\eta_2)\sigma^0}{18\eta_2}.$
Then \eqref{eq:min-nonconvex} can be derived by finding contradiction if assuming it does not hold.

More generally, for any summable sequence $\sum_k a_k<\infty$, if there exists $e>0$ such that
$\min_{K_i< k\le K_{i+1}} a_k\ge\frac{e}{K_i}$ for increasing integers $\{K_i,i=1,\ldots,\infty\},$
then $\sum_k a_k\ge\min_{K_i< k\le K_{i+1}} a_k=\infty$ is contradictory to the assumption of summable sequence. Therefore, for any $\epsilon>0$, $\min_{1\le k\le K} a_k<\frac\epsilon{K}$ holds except for finite choices of $K$, which is equivalent to say $\min_{1\le k\le K} a_k=o(\frac1K)$.

{\bf Step 3. Communication-saving.} From \eqref{eq:sum-nablaF}, we have
$$\min_{0\le k\le K-1}\|\nabla F(\mathbf{x}^k)\|^2\le\frac6{\alpha K}N_{CSGD}.$$

Following the above three steps, SGD analogously satisfies
$$\min_{0\le k\le K-1}\|\nabla F(\bar{\mathbf{x}}^k)\|^2\le\frac6{\alpha K}N_{SGD}$$
with $N_{SGD}=V^0+\frac{7\alpha(1-\eta_2)\sigma^0}{18\eta_2}.$

In CSGD, the property that the number of communications is at most $\lceil\frac{K-1}{D}\rceil+1$
after $K$ iterations for each worker still holds in this nonconvex case, since the observation \eqref{eq:contra} can also be derived as we focus on the situation that
$$\|\nabla_m^k-\nabla F_m(\mathbf{x}^{k-1})\|^2\le \frac{\sigma^{k+D}}{6M^2},$$
which is the same as in Theorem \ref{thm:sav}.
On the other hand, to reach the same accuracy after running CSGD for $K$ iterations, the iteration complexity of SGD is $$\bar{K}=\lceil (K-1)\frac{N_{\rm SGD}}{N_{\rm CSGD}}\rceil.$$
Then in order to have less communication complexity,
$\lceil\frac{K-1}{D}\rceil+1\le\bar{K}.$
Since
$\lceil\frac{K-1}{D}\rceil+1\le\frac{K-1}{D}+2$ and
$\bar{K}\ge (K-1)\frac{N_{\rm SGD}}{N_{\rm CSGD}},$
it suffices to have
$$\frac1D\le\frac{N_{\rm SGD}}{N_{\rm CSGD}},\quad
K\ge1+\frac2{\frac{N_{\rm SGD}}{N_{\rm CSGD}}-\frac1D}.$$
Note that $\frac{N_{\rm SGD}}{N_{\rm CSGD}}=\frac{V^0+\frac{7\alpha(1-\eta_2)\sigma^0}{18\eta_2}}{V^0+\frac{67\alpha(1-\eta_2)\sigma^0}{18\eta_2}}\ge\frac7{67}$,
then CSGD saves communication if $$D\ge\lceil\frac{67}{7}\rceil=10,\quad K\ge1+\lceil\frac2{\frac7{67}-\frac1{10}}\rceil=448,$$
which completes the proof.
\end{proof}

\bibliographystyle{unsrt}
\bibliography{ref,uclaopt}

\begin{thebibliography}{10}

\bibitem{dean2012}
Jeffrey Dean, Greg Corrado, Rajat Monga, Kai Chen, Matthieu Devin, Mark Mao,
  Marc'aurelio Ranzato, Andrew Senior, Paul Tucker, Ke~Yang, Quoc~V. Le, and
  Andrew~Y. Ng.
\newblock Large scale distributed deep networks.
\newblock In {\em Proc. Advances in Neural Info. Process. Syst.}, pages
  1223--1231, Lake Tahoe, NV, 2012.

\bibitem{Konecny2016}
Jakub Konecn{\'{y}}, H.~Brendan McMahan, Felix~X. Yu, Peter Richt{\'{a}}rik,
  Ananda~Theertha Suresh, and Dave Bacon.
\newblock Federated learning: Strategies for improving communication
  efficiency.
\newblock {\em arXiv preprint:1610.05492}, October 2016.

\bibitem{Brendan2017}
Brendan McMahan, Eider Moore, Daniel Ramage, Seth Hampson, and Blaise~Aguera
  y~Arcas.
\newblock Communication-efficient learning of deep networks from decentralized
  data.
\newblock In {\em Proc. Intl. Conf. AI. Stat.}, volume~54, pages 1273--1282,
  Apr 2017.

\bibitem{smith2017}
Virginia Smith, Chao-Kai Chiang, Maziar Sanjabi, and Ameet~S. Talwalkar.
\newblock Federated multi-task learning.
\newblock In {\em Proc. Advances in Neural Info. Process. Syst.}, pages
  4427--4437, Long Beach, CA, December 2017.

\bibitem{Brendan2018}
Sebastian Caldas, Jakub Konecn{\'{y}}, H.~Brendan McMahan, and Ameet Talwalkar.
\newblock Expanding the reach of federated learning by reducing client resource
  requirements.
\newblock {\em arXiv preprint:1812.07210}, December 2018.

\bibitem{Nedic2017}
Angelia {Nedić}, Alex {Olshevsky}, and Michael~G. {Rabbat}.
\newblock Network topology and communication-computation tradeoffs in
  decentralized optimization.
\newblock {\em Proc. of the IEEE}, 106(5):953--976, May 2018.

\bibitem{Berahas2018}
Albert~S. {Berahas}, Raghu {Bollapragada}, Nitish~Shirish {Keskar}, and Ermin
  {Wei}.
\newblock Balancing communication and computation in distributed optimization.
\newblock {\em IEEE Trans. Automat. Control.}, September 2018.

\bibitem{gg2016}
Georgios~B. Giannakis, Qing Ling, Gonzalo Mateos, Ioannis~D. Schizas, and Hao
  Zhu.
\newblock Decentralized learning for wireless communications and networking.
\newblock In {\em Splitting Methods in Communication and Imaging, Science and
  Engineering}. Springer, New York, NY, 2016.

\bibitem{zhang2013}
Yuchen Zhang, John~C. Duchi, and Martin~J. Wainwright.
\newblock Communication-efficient algorithms for statistical optimization.
\newblock {\em J. Machine Learning Res.}, 14(11):3321--3363, November 2013.

\bibitem{Li2014}
Mu~Li, David~G. Andersen, Alexander~J. Smola, and Kai Yu.
\newblock Communication efficient distributed machine learning with the
  parameter server.
\newblock In {\em Proc. Advances in Neural Info. Process. Syst.}, pages 19--27,
  Montreal, Canada, December 2014.

\bibitem{jordan2018}
Michael~I. Jordan, Jason~D. Lee, and Yun Yang.
\newblock Communication-efficient distributed statistical inference.
\newblock {\em J. American Statistical Association}, 2018.

\bibitem{Tianbao2013}
Tianbao Yang.
\newblock Trading computation for communication: Distributed stochastic dual
  coordinate ascent.
\newblock In {\em Proc. Advances in Neural Info. Process. Syst.}, pages
  629--637. Lake Tahoe, NV, December 2013.

\bibitem{Jaggi2014}
Martin Jaggi, Virginia Smith, Martin Takac, Jonathan Terhorst, Sanjay Krishnan,
  Thomas Hofmann, and Michael~I. Jordan.
\newblock Communication-efficient distributed dual coordinate ascent.
\newblock In {\em Proc. Advances in Neural Info. Process. Syst.}, pages
  3068--3076. Montreal, Canada, December 2014.

\bibitem{Tang2018}
Hanlin Tang, Ce~Zhang, Shaoduo Gan, Tong Zhang, and Ji~Liu.
\newblock Decentralization meets quantization.
\newblock {\em arXiv preprint:1803.06443}, March 2018.

\bibitem{Rao2019}
Milind Rao, Stefano Rini, and Andrea Goldsmith.
\newblock Distributed convex optimization with limited communications.
\newblock {\em arXiv preprint:1810.12457}, October 2018.

\bibitem{Alistarh2017}
Dan Alistarh, Demjan Grubic, Jerry Li, Ryota Tomioka, and Milan Vojnovic.
\newblock {QSGD}: Communication-efficient {SGD} via gradient quantization and
  encoding.
\newblock In {\em Proc. Advances in Neural Info. Process. Syst.}, pages
  1709--1720. Long Beach, CA, December 2017.

\bibitem{bernstein2018icml}
Jeremy Bernstein, Yu-Xiang Wang, Kamyar Azizzadenesheli, and Animashree
  Anandkumar.
\newblock {SignSGD: C}ompressed optimisation for non-convex problems.
\newblock In {\em Proc. Intl. Conf. Machine Learn.}, pages 559--568, Stockholm,
  Sweden, July 2018.

\bibitem{stich2018nips}
Sebastian~U. Stich, Jean-Baptiste Cordonnier, and Martin Jaggi.
\newblock Sparsified {SGD} with memory.
\newblock In {\em Proc. Advances in Neural Info. Process. Syst.}, pages
  4447--4458, Montreal, Canada, December 2018.

\bibitem{alistarh2018}
Dan Alistarh, Torsten Hoefler, Mikael Johansson, Nikola Konstantinov, Sarit
  Khirirat, and C{\'e}dric Renggli.
\newblock The convergence of sparsified gradient methods.
\newblock In {\em Proc. Advances in Neural Info. Process. Syst.}, pages
  5973--5983, Montreal, Canada, December 2018.

\bibitem{Garcia2013}
Eloy Garcia, Yongcan Cao, Han Yu, Panos Antsaklis, and David Casbeer.
\newblock Decentralized event-triggered cooperative control with limited
  communication.
\newblock {\em International Journal of Control}, 86(9):1479--1488, September
  2013.

\bibitem{Aji2017}
Alham~Fikri Aji and Kenneth Heafield.
\newblock Sparse communication for distributed gradient descent.
\newblock In {\em Proc. Conf. Empirical Methods Natural Language Process.},
  pages 440--445, Copenhagen, Denmark, September 2017. Association for
  Computational Linguistics.

\bibitem{Li2019}
Weiyu Li, Yaohua Liu, Zhi Tian, and Qing Ling.
\newblock {COLA: C}ommunication-censored linearized admm for decentralized
  consensus optimization.
\newblock In {\em Proc. IEEE Intl. Conf. Acoustics, Speech and Signal
  Process.}, pages 5237--5241, Brighton, England, May 2019.

\bibitem{Liu2019}
Yaohua Liu, Wei Xu, Gang Wu, Zhi Tian, and Qing Ling.
\newblock Communication-censored {ADMM} for decentralized consensus
  optimization.
\newblock {\em IEEE Trans. Sig. Proc.}, 67(10):2565--2579, May 2019.

\bibitem{lan2017}
Guanghui Lan, Soomin Lee, and Yi~Zhou.
\newblock Communication-efficient algorithms for decentralized and stochastic
  optimization.
\newblock {\em arXiv preprint:1701.03961}, January 2017.

\bibitem{Tianyi2018}
Tianyi Chen, Georgios B.~Giannakis, Tao Sun, and Wotao Yin.
\newblock {LAG: L}azily aggregated gradient for communication-efficient
  distributed learning.
\newblock In {\em Proc. Advances in Neural Info. Process. Syst.}, pages
  5050--5060. Montreal, Canada, December 2018.

\bibitem{Tianyi2018-2}
Tianyi Chen, Kaiqing Zhang, Georgios B.~Giannakis, and Tamer Başar.
\newblock Communication-efficient distributed reinforcement learning.
\newblock {\em arXiv preprint:1812.03239}, December 2018.

\bibitem{bottou2010}
L{\'e}on Bottou.
\newblock Large-scale machine learning with stochastic gradient descent.
\newblock In Yves Lechevallier and Gilbert Saporta, editors, {\em Proc. of
  {{COMPSTAT}}}, pages 177--186. {Physica-Verlag HD}, Heidelberg, 2010.

\bibitem{Bottou2018}
L\'eon {Bottou}, Frank~E. {Curtis}, and Jorge {Nocedal}.
\newblock Optimization methods for large-scale machine learning.
\newblock {\em SIAM Reviews}, 60(2):223--311, 2018.

\bibitem{Hao2019}
Hao Yu and Rong Jin.
\newblock On the computation and communication complexity of parallel sgd with
  dynamic batch sizes for stochastic non-convex optimization.
\newblock {\em arXiv preprint:1905.04346}, May 2019.

\bibitem{Karimi2016}
Hamed Karimi, Julie Nutini, and Mark Schmidt.
\newblock Linear convergence of gradient and proximal-gradient methods under
  the {Polyak-{\L}ojasiewicz} condition.
\newblock In {\em European Conf. Machine Learn. and Knowledge Discovery in
  Databases}, pages 795--811, Riva del Garda, Italy, 2016.

\bibitem{bottou2016}
L{\'e}on Bottou, Frank~E Curtis, and Jorge Nocedal.
\newblock Optimization methods for large-scale machine learning.
\newblock {\em arXiv preprint:1606.04838}, June 2016.

\bibitem{lecun1998mnist}
Yann LeCun, Corinna Cortes, and Christopher~JC Burges.
\newblock The {MNIST} database.
\newblock {\em \url{http://yann.lecun.com/exdb/mnist}}, 1998.

\bibitem{covtype}
Dheeru Dua and Casey Graff.
\newblock {UCI} machine learning repository, 2017.

\bibitem{cifar10}
Alex Krizhevsky.
\newblock Learning multiple layers of features from tiny images.
\newblock {\em University of Toronto}, 1(4):7, 05 2012.

\bibitem{Kaiming2018}
Kaiming He, Xiangyu Zhang, Shaoqing Ren, and Jian Sun.
\newblock Deep residual learning for image recognition.
\newblock In {\em Proc. IEEE Conf. on Computer Vision and Pattern Recognition},
  pages 770--778, June 2016.

\bibitem{Durrett2010}
Rick Durrett.
\newblock {\em Probability : {T}heory and {E}xamples}.
\newblock Cambridge University Press, New York, NY, 4th edition, 2010.

\end{thebibliography}

\end{document}